\documentclass{article}

\usepackage{microtype}
\usepackage{graphicx}
\usepackage{booktabs}
\usepackage{hyperref}
\usepackage{subfig}
\usepackage{tikz}
\usetikzlibrary{arrows}
\usetikzlibrary{decorations.pathreplacing}

\usepackage{amsmath,amsfonts,bbm,amsthm,stmaryrd,amssymb}
\usepackage{xspace}
\usepackage{xcolor,colortbl}

\newtheorem{theorem}{Theorem}
\newtheorem{definition}{Definition}
\newtheorem{lemma}{Lemma}

\newcommand{\parfrac}[2]{\paran{\frac{#1}{#2}}}
\newcommand{\paran}[1]{\left( #1 \right)}
\newcommand{\pnorm}[2]{\left|\left| {#2} \right|\right|_{#1}}
\newcommand{\vect}[1]{\left[ #1 \right]}
\newcommand{\abs}[1]{\left| #1 \right|}

\newcommand{\logloss}{\textsc{LLoss}\xspace}
\newenvironment{CompactEnumerate}{
\begin{list}{\roman{enumi}.}{%
\usecounter{enumi}
\setlength{\leftmargin}{12pt}
\setlength{\itemindent}{5pt}
\setlength{\topsep}{1pt}
\setlength{\itemsep}{-1ex}
}}
{\end{list}}

\usepackage[accepted]{icml2021}

\icmltitlerunning{Label Inference Attacks from Log-Loss Scores}

\begin{document}

\twocolumn[
\icmltitle{Label Inference Attacks from Log-loss Scores}
\begin{icmlauthorlist}
\icmlauthor{Abhinav Aggarwal}{to}
\icmlauthor{Shiva Prasad Kasiviswanathan}{to}
\icmlauthor{Zekun Xu}{to}
\icmlauthor{Oluwaseyi Feyisetan}{to}
\icmlauthor{Nathanael Teissier}{to}
\end{icmlauthorlist}

\icmlaffiliation{to}{Amazon}

\icmlcorrespondingauthor{Abhinav Aggarwal}{aggabhin@amazon.com}
\icmlcorrespondingauthor{Shiva Prasad Kasiviswanathan}{kasivisw@amazon.com}
\vskip 0.3in
]

\printAffiliationsAndNotice{}

\begin{abstract}
Log-loss (also known as cross-entropy loss) metric is  ubiquitously used across machine learning applications to assess the performance of classification algorithms. In this paper, we investigate the problem of inferring the labels of a dataset from single (or multiple) log-loss score(s), without any other access to the dataset. Surprisingly, we show that for any finite number of label classes, it is possible to accurately infer the labels of the dataset from the reported log-loss score of a single carefully constructed prediction vector if we allow arbitrary precision arithmetic. Additionally, we present label inference algorithms (attacks) that succeed even under addition of noise to the log-loss scores and under limited precision arithmetic.
All our algorithms rely on ideas from number theory and combinatorics and require no model training. We run experimental simulations on some real datasets to demonstrate the ease of running these attacks in practice.
\end{abstract}

\section{Introduction}
\label{sec:introduction}
Log-loss (a.k.a.\ cross-entropy loss) is an important metric of choice in evaluating machine learning classification algorithms. Log-loss is based on prediction probabilities where a lower log-loss value means better predictions. Therefore, log-loss is useful to compare models not only on their output but on their probabilistic outcome. 

Let $[K]=\{1,\dots,K\}$ be a set of label classes and consider a dataset of $N$ datapoints with true labels $\sigma \in [K]^N$. The $K$-ary log-loss score takes as input $\sigma$ and a matrix in $[0,1]^{N \times K}$ of prediction probabilities, where the $i$th row is the vector of prediction probabilities $u_{i,1},\dots,u_{i,K}$ (with $\sum_{k \in [K]} u_{i,k} = 1$) for the $i$th datapoint on the $K$-classes.
\begin{definition}[$K$-ary log-loss Score]\label{def:log_loss_multi}
	Let $\mathbf{u} \in [0,1]^{N \times K}$ be a matrix such that for all $k \in [K]$ and $i \in [N]$, it holds that $\sum_{i=1}^N u_{i,k} = 1$. Let $\sigma \in [K]^N$ be a labeling. Then, the $K$-ary log-loss (or, cross-entropy loss) on $\mathbf{u}$ with respect to $\sigma$, denoted by $\logloss\paran{\mathbf{u}; \sigma}$, is defined as follows:
	\begin{align*}
		\logloss\paran{\mathbf{u}; \sigma} := \frac{-1}{N}\sum_{i=1}^N \sum_{k=1}^K \ \Big([\sigma_i = k] \cdot \ln u_{i,k} \Big),
	\end{align*}
	where $[\sigma_i = k] = 1$ if $\sigma_i = k$ and $0$, otherwise.
\end{definition}
Typically, the $\mathbf{u} \in [0,1]^{N \times K}$ is generated by a ML model. We note that some texts use the unscaled version of log-loss and that our algorithms can be easily extended to these variants (see~\cite{murphy2012machine} for more details on log-loss).

A special case is when $K=2$ (binary labels) in which case the above definition reduces to the following simpler form.
\begin{definition}[Binary log-loss Score]\label{def:log_loss_binary} 
	Given a vector $\mathbf{u} = (u_1,\dots,u_N) \in [0,1]^N$ and a labeling $\sigma \in \{0,1\}^N$, the log-loss on $\mathbf{u}$ with respect to $\sigma$, denoted by $\logloss\paran{\mathbf{u}; \sigma}$, is defined as follows:$$\logloss\paran{\mathbf{u}; \sigma} := \frac{-1}{N}\ln \paran{\prod_{i=1}^N u_i^{\sigma_i}(1-u_i)^{1-\sigma_i}}.$$
\end{definition}

\begin{table*}[t]
    \centering
    \caption{Overview of our results for binary label inference. Here, $N \ge 1$ is the number of labels to be inferred. We present attacks under both arbitrary and bounded precision arithmetic models (see Section~\ref{sec:model} for a comparison of these models). The $\tau$-accurate means that the error on the responses are bounded by $|\tau|$. The fourth column represents the number of arithmetic operations needed at the adversary. 
    All our adversaries are polynomial time  except for the third row. 
    }
    \begin{tabular}{|>{\centering\arraybackslash}m{4.8cm}|c|c|c|c|}
        \hline
         \textbf{Amount of Noise in Responses}& \textbf{Precision} & \textbf{\# Log-loss Queries} & \textbf{\#Arithmetic Operations} & \textbf{Reference}  \\ \hline \hline 
         
         No noise & Arbitrary  & 1 & $O\paran{N}$ & Theorem~\ref{thm:basic_log_loss_attack}\\ \hline
         
         No noise & $\phi$-bits  & $\Theta\paran{1+N\phi2^{-\phi/4}}$ & $O\paran{N}$ & Algorithm~\ref{alg:multi_query}\\ \hline
         
         \shortstack[c]{$\tau$-accurate  } & Arbitrary  & 1 & $O(2^N)$ & Algorithm~\ref{alg:bounded_noise_exponential}
         \\ \hline
         
         $\tau$-accurate & $\phi$-bits  & $O\paran{\frac{N}{\log N} + \frac{N}{\log \paran{\phi/N\tau}}}$ & $O\paran{\frac{poly\paran{N, \phi/\tau}}{\log (\phi/N\tau)}}$ & Algorithm~\ref{alg:reconstruction_attack}\\ \hline
         
    \end{tabular}
    \label{tab:overview_results}
\end{table*}

Given its preeminent role in evaluating machine learning models, especially in the neural network literature~\citep{goodfellow2016deep}, an important question arises is whether it is possible to ``exploit'' the log-loss score. In particular, we ask whether the knowledge of log-loss scores leaks information about the true labels $\sigma$. We answer the question in affirmative by showing that for any finite number of label classes, it is possible to infer all of the dataset labels from just the reported log-loss scores if the prediction probability vectors are carefully constructed and this can be done {\em without} any model training. In fact, we present stronger inference attacks, that succeed even when the log-loss scores are perturbed by noise. 

Our inference attacks have important consequences:
\begin{list}{{\bf (\roman{enumi})}}
	{\usecounter{enumi}		\setlength{\leftmargin}{11pt}
 	\setlength{\listparindent}{-1pt}
	\setlength{\parsep}{-1pt}	
}
\item Integrity of Machine Learning Competitions: Many machine learning (data mining) competitions such as those organized by Kaggle\footnote{https://www.kaggle.com/}, KDDCup\footnote{https://www.kdd.org/kdd-cup} and ILSVRC Challenge~\footnote{http://www.image-net.org/challenges/LSVRC/} use log-loss as their choice of evaluation metric. In particular, it is common in these competitions that the  quality of a participants' solution to be assessed through a log-loss on an unknown test dataset. Our results demonstrate that an unscrupulous participant can game this system, by using the log-loss score to learn the test set labels, and thereby constructing a fake but perfect classifier with zero test error. The simplicity and efficacy of our proposed attacks make this issue a real concern.\!\footnote{If needed, an attacker can  obfuscate the prediction vectors needed for our attacks in its ML models. We do not focus on this aspect here.}

\item Privacy Concerns: ML models are regularly trained on sensitive datasets. Imagine an adversary who can ask log-loss scores for supplied prediction vectors. While at the onset the log-loss being a non-linear function, might look innocuous to release, our results show the extent of information leakage from these scores. In fact, we get a perfect reconstruction, a stronger privacy violation than that achieved by the {\em blatant non-privacy} notion~\citep{dinur2003revealing}, which only requires a large fraction of the sensitive data to be reconstructed.
\end{list}

\noindent\textbf{Overview of Our Results.} We present multiple inference attacks from log-loss scores under various constraints such as bits of arithmetic precision, noise etc. All our attacks operate only based on the ability of an adversary to query the log-loss scores on the chosen prediction vectors,  without any access to the feature set or requiring any model training. All  our inference attacks are also completely agnostic of the underlying classification task. 

Our primary focus in this paper is on the binary label case.
An overview of our main results for the binary label inference, that we discuss below, is summarized in Table~\ref{tab:overview_results}. We start with the simplest setting, where the log-loss scores are observed in the raw (without any noise). If the adversary has the ability to perform arbitrary precision arithmetic, we show that with just one log-loss query, an adversary can recover all the labels. We extend this result to the case where the adversary performs $\phi$-bits precision arithmetic, and show that the labels can be recovered with $\Theta\paran{1+N\phi2^{-\phi/4}}$ log-loss queries. Both these attacks require only a polynomial-time adversary and also extend to the multiclass case.

We then move on to the more challenging setting where the scores can be perturbed with noise before the adversary observes them. Assuming that the responses are $\tau$-accurate (i.e., within error $\pm \tau$), we show that an adversary can still recover all the labels correctly in the arbitrary precision model with just one query but now with exponential time. Interestingly, this holds independent of $\tau$. The construction here uses large numbers (that are doubly-exponential in $N$) that our lower-bounds suggest are unfortunately unavoidable.
In the $\phi$-bits precision model, we present a polynomial-time adversary that requires $O\paran{N/\log N + N/\log \paran{\phi/N\tau}}$ queries.

Next, we present extensions of these attacks to other interesting noise settings such as randomly generated noise and multiplicative noise, and show how to recover labels in those settings (see Section~\ref{sec:extension_other_noise}). Finally, in Section~\ref{sec:expts}, we present experiments that demonstrate the remarkable effectiveness and speed of these attacks on real and simulated datasets.

We note that while the techniques for label inference in the noised case will also hold for the (raw) unnoised case, we discuss the later separately to capture some key ideas behind our constructions. Moreover, our construction for inference from raw scores has some advantages, it uses a fewer number of queries and can be easily extended to the multiclass setting. 

\noindent\textbf{Overview of Our Techniques.} Our attacks are based on a variety of number-theoretic and combinatorial techniques that we briefly summarize here.
\begin{itemize}
    \itemsep0em
    \item For the case where log-loss scores are returned without any noise, we use the \textit{Fundamental Theorem of Arithmetic}~\cite{hardy1979statement}, which states that every positive integer has a unique prime factorization. We assign powers of distinct primes to different datapoints in a way that all labels can be recovered in a single query, for both binary as well as the multiclass case. Moreover, since the list of primes is well-known and only a function of the number of datapoints, recovery of labels from the observed scores is efficient assuming arbitrary-precision arithmetic.
    \item To adapt our construction above (in the no-noise setting) to the case of bounded floating-point precision, we bound the number of log-loss queries using the well-known \textit{Prime Number Theorem}~\cite{poussin1897recherches,hadamard1896distribution}, which provides an asymptotic growth rate for the size of prime numbers. Our construction achieves label inference in an optimal number of queries, which we prove by providing matching upper and lower bounds.
    \item For the case of label inference from noised scores, we use a different attack strategy. Here, we reduce this problem to the \textit{construction of sets with distinct subset sums}. For the lower-bound here, we build upon the classic result (first conjectured) by Euler (and later proved by~\cite{benkoski1974weird,frenkel1998integer}), which bounds the size of the largest element in such sets with integer elements. 
\end{itemize}

\section{Problem Definition and Setting}
\label{sec:model}
We begin by formally defining our model of computation. In this paper, we discuss \emph{perfect} label inference problem, which refers to inferring \emph{all} the labels. We formally define label inference under different constraints in subsequent sections. We refer to the entity that runs this inference as the \emph{adversary}.

Throughout the paper, unless otherwise stated, we focus on the binary label case.
We refer to the vector $\mathbf{u}$ in Definition~\ref{def:log_loss_binary} as the \emph{prediction vector} used by the adversary. The key ideas in our label inference algorithms are best explained by describing a vector $\mathbf{v} = (v_1,\dots,v_N) \in \mathbb{R}^N$ and then constructing the prediction vector $\mathbf{u} = f(\mathbf{v}) := \left[ f(v_1),\dots,f(v_N) \right]$, where $f(x) = \frac{x}{1+x}$. 

For notational convenience, we define:
\begin{align} \label{eqn:equiv}
\mathcal{L}_{\mathbf{v}}\paran{\sigma} := \logloss\paran{f(\mathbf{v}),\sigma} = \logloss\paran{\mathbf{u},\sigma}.
\end{align}
Note that for any $\sigma \in \{0,1\}^N$, a simple algebraic manipulation of $\mathcal{L}_{\mathbf{v}}\paran{\sigma}$ gives the following:
\begin{align}
    \mathcal{L}_{\mathbf{v}}\paran{\sigma}  
    = \frac{-1}{N} \ln \parfrac{\prod_{i: \sigma_i = 1} v_i}{(1+v_1)\dots(1+v_N)} \label{eq:loss_expression}.
\end{align}
We will repeatedly refer to this form in our constructions. We are interested in vectors $\mathbf{v}$ for which the function $\mathcal{L}_{\mathbf{v}}$ is injective (i.e., a 1-1 correspondence). This injection will allow the adversary to ensure that the true labeling can be unambiguously recovered from the observed loss score.

\noindent\textbf{Models of Computation.}
We present our results in two models of arithmetic computation. The first model assumes \textit{arbitrary precision arithmetic}, which allows precise arithmetic results even with very large numbers. We refer to this as the 
$\mathsf{APA}$ model. While this model results in considerably slower arithmetic~\citep{brent2010modern}, it helps an adversary to perform label inference with fewer queries.

The second model is the more standard \textit{floating point precision model}, where the arithmetic is constrained by limited precision. We denote this model as $\mathsf{FPA}(\phi)$, where $\phi$ represents the number of bits of precision. We assume the following abstraction for the format for representing numbers in this model: 1 bit for sign, $(\phi-1)/2$ bits for the exponent and $(\phi-1)/2$ bits for the fractional part (mantissa).\footnote{More generally, one could allocate $\phi_a$ bits for the exponent and $\phi_b$ bits for the fractional part where $\phi_a + \phi_b = \phi-1$. This is the setting in our experiments.} This allows representing all numbers between $-2^{(\phi-1)/2}$ to $+2^{(\phi-1)/2}$, with a resolution of $2^{-(\phi-1)/2}$. The floating point precision model allows for more efficient arithmetic operations~\citep{brent2010modern}. We refer the reader to Chapter 4 in~\cite{ knuth2014art} for a detailed discussion on designing algorithms for standard arithmetic in these models.

\noindent\textbf{Threat Model.} 
We assume that the adversary sends a prediction vector $\mathbf{u}$ to a machine (server) that holds the (private) dataset $\sigma \in \{0,1\}^N$ (also called \emph{labeling}) and gets back $\logloss\paran{\mathbf{u},\sigma}$.
We assume that the loss is computed on all labels in $\sigma$ and that $N$ is known to the adversary. In the setting where the scores are noised, we assume that the adversary knows an upper bound on the resulting error. Often the former can be inferred from the knowledge of the precision on the machine that returns the loss score to the adversary. Finally, we assume that the adversary can make multiple queries with different prediction vectors and obtain the corresponding loss scores. Since this query access can be limited in practical settings, we optimize the number of queries required by our inference algorithms and prove formal lower bounds.

We refer to the adversary as a {\em polynomial-time} adversary if it is restricted to only polynomial-time computations, otherwise we refer to it as an \emph{exponential-time} adversary.

\noindent\textbf{Related Work.} \citet{whitehill2018climbing} initiated the study of how log-loss scores can be exploited in ML competitions, which was optimized for single-query inference in~\cite{aggarwal2020logloss}. However, the attack in~\cite{whitehill2018climbing} constructs prediction vectors (which they call probe matrices) by heuristically solving a min-max optimization problem in a space that is exponentially large in the number of labels their algorithm infers in a single query. This heuristic is based on a Monte-Carlo simulation, which severely limits the scalability of their algorithm to arbitrary large datasets (and/or to arbitrarily many number of classes). In contrast, our construction is simple and practical, which makes our attack efficient (see Section~\ref{sec:expts} for details). Additionally, the algorithms in both~\citep{whitehill2018climbing,aggarwal2020logloss} cannot be extended to the noised case -- the attacks by~\cite{aggarwal2020logloss}, in particular, use a single log-loss query and hence, cannot be run in the finite precision setting. Additionally, the use of Twin Primes in~\citep{aggarwal2020logloss} is missing a discussion on efficiently constructing such primes for arbitrarily large datasets. 

Label inference attacks based on other metrics such as AUC scores are also known~\citep{whitehill2016exploiting,matthews2013examination}, but we do not know of any connection between these and our setting. In a recent work,~\cite{blum2015ladder} demonstrate general techniques for safeguarding leaderboards in Kaggle-type competition settings against an adversarial boosting attack, in which the attacker observes loss scores on randomly generated prediction vectors to generate a labeling which, with probability $2/3$, gives a low loss function. 
% However, exact label recovery, especially for the multiclass setting and in the noised setting, still remained an open problem, that we address here. 

The constructions introduced in this paper are related to the ideas prevalent in the coding theory literature.\!\footnote{This was pointed to us by the anonymous ICML reviewers.} The idea of designing the prediction vector ($\mathbf{u}$) can be viewed as constructing a coding scheme, whose input is the true labels, with the goal of recovering (decoding) the true labels after passing it through the log-loss function (which acts as the noisy channel). In particular, our constructions in Section~\ref{sec:bounded_noise}, have parallels to coding schemes based on Sidon sequences~\citep{o2004complete} and Golomb rulers~\citep{robinson1967class}. We believe that better label inference attacks could be designed by further exploring this connection with the coding theory literature.
% coding scheme that takes true true labels and passes

% \abhi{Shiva to add details on coding theory. Copying text from NeurIPS submission $\Longrightarrow$} An alternate view of our results can be seen through the lens of information theory. In the standard channel coding problem, the goal is to construct a coding scheme (in our case, the prediction vector) that takes a message (true labels), passes it though a noisy channel (loss function), to produce an output which is then used to infer the message (label inference mechanism). A standard coding theory goal is to have a high minimum distance between the codewords (separability). Our idea for using sets with distinct subset sums in our construction is motivated by similar concepts appearing in information theory, e.g., coding schemes based on Sidon sequences~\citep{o2004complete} and Golomb rulers~\citep{robinson1967class}. 

\noindent\textbf{Additional Notation.} We will denote by $[n] = \{1,2,\dots,n\}$ and use $\mathbb{R}$ for real numbers, $\mathbb{Z}$ for integers, and $\mathbb{Z}^{+}$ for the set of positive integers. For any vector $v = [v_1,\dots,v_n]$, we use $v[:a] = [v_1,\dots,v_a]$ for $a \in [n]$. Unless specified, all logarithms use the natural base ($e$) and $p_1,p_2,\dots$ will denote the primes ($p_i$ being the $i^{th}$ prime).

\section{Label Inference from Raw Scores}
\label{sec:rawScores}
We begin our discussion with label inference from scores that are reported without any noise added. We will first assume arbitrary  precision arithmetic to explain the key idea behind our construction, and then extend the discussion to the case of floating-point precision.  Missing details from this section are collected in the Appendix~\ref{app:rawScores}. 

We begin by formally defining the label inference problem.
\begin{definition}
    \label{def:label_inference_no_nois}
    Let $\sigma \in \{0,1\}^N$ be an (unknown) labeling. The label inference problem is that of recovering $\sigma$ given $\logloss\paran{\mathbf{u}_1; \sigma},\dots,\logloss\paran{\mathbf{u}_M; \sigma}$. Here, $M$ is the number of queries and  $\mathbf{u}_i\in [0,1]^N$ are the prediction vectors.
    % \emph{Label Inference Problem} is defined as the problem of constructing a $\tau$-robust vector for any given $N, \tau > 0$.
\end{definition}

\subsection{Single Query Label Inference under Arbitrary Precision with Polynomial-time Adversary}
For our first result, we show that it is possible to extract all ground truth labels using just one query in the arbitrary precision ($\mathsf{APA}$) model. Our key tool is the Fundamental Theorem of Arithmetic~\cite{hardy1979statement}, which states that every integer has a unique prime factorization. Recall that from Equation~\eqref{eqn:equiv}, recovering $\sigma$ from $\logloss\paran{\mathbf{u};\sigma}$ is equivalent to recovering $\sigma$ from $\mathcal{L}_{\mathbf{v}}(\sigma)$. Our construction is described below.

\begin{theorem}\label{thm:basic_log_loss_attack}
There exists a polynomial-time adversary for the single-query label inference problem in the $\mathsf{APA}$ model.
% using only a single log-loss query.
\end{theorem}
\begin{proof}
Let $\sigma$ denote the (unknown) labeling. Define $\mathbf{v} = [p_1,\dots,p_N]$ and $T = \prod_{i=1}^N (1+p_i)$. Observe that the terms in Equation~(\ref{eq:loss_expression}) can be re-arranged to give $\prod_{i: \sigma_i = 1} p_i = T\exp\paran{-N\mathcal{L}_{\mathbf{v}}\paran{\sigma}}$. This gives the required injection since there is a unique product of primes for any given value of the right hand side of this equation. Moreover, the primes in this product uniquely define which elements in $\sigma$ have label $1$, since we use a distinct prime for each $i$.     
\end{proof}
As an example, for $N = 5$, let $\mathbf{v} = [2,3,5,7,11]$. Suppose the true labeling is $[0,1,1,0,1]$. Then, the adversary observes $\mathcal{L}_{\mathbf{v}}\paran{\sigma} = \frac{1}{5}\ln \parfrac{2304}{55}$ (obtained by plugging in $\mathbf{v}$ and $\sigma$ in Equation~\ref{eq:loss_expression}). For reconstructing the labels, observe that $T = 3 \times 4 \times 6 \times 8 \times 12 = 6912$, so that all we need is to compute primes that divide $T\exp\paran{-N\mathcal{L}_{\mathbf{v}}\paran{\sigma}} = 165 = 3 \times 5 \times 11$. This tells us that only the labels for the second, third and fifth datapoints must be 1, which is indeed true. 

We note that the construction above is not unique -- all the steps in the proof would go through if we replaced $\mathbf{v}$ with \emph{any} vector containing distinct primes (or even mutually co-prime numbers) by the unique factorization property. Moreover, since the adversary decides what primes go inside $\mathbf{v}$, it only takes $O(N)$ time to determine all the factors in the product above (e.g., by checking each $p_i$ one by one). We further note that our proof assumes that $Te^{-N\mathcal{L}_{\mathbf{v}}\paran{\sigma}}$ can be written precisely and unambiguously as an integer for any $\mathbf{v}$ and $\sigma$. This requires arbitrary precision. In practical scenarios, however, this may not be the case and hence, only a few labels may be correctly inferred. We discuss our construction for exact inference in this case next.

\begin{algorithm}[t]
\caption{Label Inference with No Noise in the $\mathsf{FPA}(\phi)$ Model (Polynomial Adversary)}
 \label{alg:multi_query}
 \begin{algorithmic}[1]
  \STATE \textbf{Input:} $N$ (length of vector), $\phi$ (bits of precision)
  \STATE \textbf{Output:} Labeling $\hat{\sigma} \in \{0,1\}^N$
  \STATE Initialize $\hat{\sigma} \gets [0,\dots,0]$.
  \STATE Let $m$ be the largest integer for which $p_m \leq 2^{(\phi-5)/4}$.
%   If $m \geq \lceil \log N\rceil$, then use $m = \lceil \log N\rceil$.
  \FOR{$k$ \textbf{in} $\{1,2,\dots,\lceil N/m \rceil\}$}
    \STATE Set $\mathbf{v}^{(k)}$ with $\mathbf{v}^{(k)}_j = p_r$ if $j =
                (k-1)m + r$ for some $r\in[m]$. Else, set $\mathbf{v}^{(k)}_j = 1$.
    
    % Set $\mathbf{v}^{(k)} \gets \left[ \underbrace{1,\dots,1}_{(k-1)m\ \text{times}}, p_1,\dots,p_m,1\dots,1 \right]$, where $p_i$ is the $i^{th}$ prime. The number of 1's at the end should be enough to ensure $|\mathbf{v}^{(k)}| = N$. Truncate elements from the right if $|\mathbf{v}^{(k)}| > N$ (without the trailing 1's).
    \STATE Obtain the loss $\ell^{(k)}$ using $\mathbf{u}^{(k)} = f\paran{\mathbf{v}^{(k)}}$ as the prediction vector.
    \STATE Let $P^{(k)} = [p_1,\dots,p_{|P^{(k)}|}]$ be the set of primes inside $\mathbf{v}^{(k)}$ and $\alpha^{(k)} = \prod_{p \in P^{(k)}} (1+p)$.
    \STATE Compute $q^{(k)} \gets \alpha^{(k)}e^{-N\ell^{(k)} + (N-|P^{(k)}|)\ln 2}$.
    \STATE For $j \in \{1,\dots,|P^{(k)}|\}$, if $p_j$ divides $q^{(k)}$, then set $\hat{\sigma}_{(k-1)m + j} \gets 1$.
  \ENDFOR
  \STATE \textbf{Return} $\hat{\sigma}$.
 \end{algorithmic}
\end{algorithm}
%This can further be generalized to vectors of any collection of $N$ mutually co-prime positive integers, where we say that $m$ and $n$ are co-prime if their greatest common divisor is 1.
% The following corollary of Theorem~\ref{thm:basic_log_loss_attack} establishes this result.

% \begin{corollary}\label{cor:prime_leaky_vectors}
% Let $\mathbf{v} = \left[ v_1,\dots,v_N \right] \in (\mathbb{Z}^{+})^N$ be such that $gcd(v_i,v_j) = 1$ for all distinct $i,j \in [N]$. Then, $\mathbf{v}$ is leaky.
% \end{corollary}
% \begin{proof}
% From the proof of Theorem~\ref{thm:basic_log_loss_attack}, we have shown that $\prod_{i: \sigma_i = 1} v_i = T\exp\paran{-N\mathcal{L}_{\mathbf{v}}\paran{\sigma}}$ for any given $\sigma \in \{0,1\}^N$, where $T = (1+v_1)\dots(1+v_N)$. Now, since all $v_i$'s are mutually co-prime, they each have a distinct prime factorization (with no prime factors common with other elements in $\mathbf{v}$). Thus, the product on the left of this equation has a unique prime factorization, which can then be exploited to infer which elements in $\sigma$ have value one.

\subsection{Label Inference under Bounded Precision with Polynomial-time Adversary}
To work in the more realistic scenario of bounded precision within the $\mathsf{FPA}(\phi)$ model, we are restricted in our choice of primes since the primes from Theorem~\ref{thm:basic_log_loss_attack} can get very large (as $N$ increases). We handle this using multiple queries, inferring only a few labels at a time (details outlined in Algorithm~\ref{alg:multi_query}). In each iteration, the prediction vectors use primes for the bits not yet inferred, and the remaining entries are kept fixed. This way the remaining entries contribute a fixed amount to the loss, which can be subtracted at the time of label inference. Thus, using a smaller number of primes allows us to work with the available precision budget. 
\begin{theorem}\label{thm:basic_log_loss_attack_FPA}
Let $\phi \geq 9$. There exists a polynomial-time adversary (from Algorithm~\ref{alg:multi_query}) for the label inference problem in the $\mathsf{FPA}(\phi)$ model using $\Theta\paran{1+N\phi2^{-\phi/4}}$ queries.
\end{theorem} 

\begin{proof}
   We begin by proving the upper bound on the number of queries. We refer to Algorithm~\ref{alg:multi_query} for our proof. 
   
   Let $\sigma$ denote the true labeling. Fix some $m \in [N]$ and without loss of generality, assume that $N$ is a multiple of $m$, so that $\lceil N/m \rceil = N/m$. Then, in the $k^{th}$ iteration of the for-loop in Algorithm~\ref{alg:multi_query}, the prediction vector $\mathbf{u}^{(k)}$ uses $m$ primes $P^{(k)} = [p_1,\dots,p_m]$, with all other entries set to 1. Since the computation of log-loss is invariant of the relative order of the datapoints (as long as it aligns with the prediction vector), it suffices to show that the first iteration of the loop ($k=1$) correctly recovers the first $m$ labels. To see this, observe that Lemma~\ref{lem:loss_vector_with_ones} gives us that the log-loss score observed on $\mathbf{u}^{(k)}$ has the following form:
   \begin{align*}
       &-N\ell^{(1)} = -N\mathcal{L}_{\mathbf{v}^{(1)}}(\sigma)\\ 
       &= \ln \paran{\prod_{\substack{i: \sigma_i = 1\\1 \leq i \leq M}} p_i} - \paran{N-M}\ln 2 - \sum_{j=1}^M \ln (1+p_j)\\
       \implies& \paran{\prod_{j=1}^M (1+p_j)}e^{-N\ell^{(1)} + (N-M)\ln 2} = \prod_{\substack{i: \sigma_i = 1\\1 \leq i \leq M}} p_i.
    \end{align*}
    % Rearranging the terms above, we obtain:
    % \begin{align*}
    %     % &e^{-N\ell^{(1)}} = \paran{\prod_{\substack{i: \sigma_i = 1\\1 \leq i \leq M}} p_i} \paran{\prod_{j=1}^M \parfrac{1}{1+p_j}}e^{-(N-M)\ln 2}\\
    %     % \implies& \paran{\prod_{j=1}^M (1+p_j)}e^{-N\ell^{(1)} + (N-M)\ln 2} = \prod_{\substack{i: \sigma_i = 1\\1 \leq i \leq M}} p_i.
    % \end{align*}
   The product term on the left is the same as $\alpha^{(1)}$ and the product term on the right allows for unambiguous label recovery (via unique factorization) in Step 10 of Algorithm~\ref{alg:multi_query}.
   
   Next, we prove that in the $\mathsf{FPA(\phi)}$ model, a polynomial-time adversary can infer at most $2^{\phi/4}/\phi$ labels per query. 
   
   To see this bound on the number of labels inferred per iteration, first observe that:
   \begin{align*}
       &\min_{\substack{i,j \in [N]\\p_i \neq p_j}} \abs{\frac{p_i}{1+p_i} - \frac{p_j}{1+p_j}} \geq \frac{p_m}{1+p_m} - \frac{p_{m-1}}{1+p_{m-1}}
       \\
       &= \frac{p_m - p_{m-1}}{(1+p_m)(1+p_{m-1})} \geq \frac{p_m - p_{m-1}}{4p_mp_{m-1}}\\
       &= \frac{1}{4}\paran{\frac{1}{p_{m-1}} - \frac{1}{p_m}} \geq \frac{1}{4}\paran{\frac{1}{p_{m}-1} - \frac{1}{p_m}} \geq \frac{1}{4p_m^2},
   \end{align*}
   where the first line follows from the fact that $p_1 < \cdots < p_m$, and the second line from $p_m \ge 3$ and $p_{m-1} \geq 2$. Thus, in the $\mathsf{FPA}(\phi)$ model, we can only use $m$ that is large enough so that the following continues to holds:
   \begin{align*}
       &\frac{\phi-1}{2} \geq \log_2(4p_m^2) = 2+2\log_2 p_m\\
       \implies& \phi \ge 5 + 4\log_2 p_m.
   \end{align*}
   From this, we obtain that the largest prime $p_m$ that can be used must be at most $2^{(\phi-5)/4}$, as mentioned in Algorithm~\ref{alg:multi_query} (which establishes the optimality of our construction). 
   
   Finally, from Lemma~\ref{lem:prime_number_theorem}, since $p_m = \Theta(m\log m)$, we obtain $m = \Theta\paran{2^{\phi/4}/\phi}$, which gives the number of queries as $\Theta\paran{N\phi2^{-\phi/4}}$ for our inference attack.  
\end{proof}

Note that the bound on the number of queries is asymptotically tight. We prove this using the Prime Number Theorem~\cite{hadamard1896distribution,poussin1897recherches}, which describes the asymptotic distribution of primes among the integers. In Section~\ref{sec:expts}, we present experimental results to show Algorithm~\ref{alg:multi_query} can be used for label inference on real datasets.
% Table~\ref{tab:experiment_no_noise} and the bottom row of Figure~\ref{fig:experiment_plots} to see the effect of using multiple queries on the inference accuracy for real-life datasets.

\subsection{Extension to the Multiclass Case} 
Our construction using primes in Theorem~\ref{thm:basic_log_loss_attack} can be extended to multiple classes as well. We now use Definition~\ref{def:log_loss_multi}. 
% Let $K$ be the number of classes. The input to the loss function are now $K$-dimensional vectors for each datapoint. In particular, let $\mathbf{v} \in [0,1]^{N \times K}$ be a matrix such that for all $k \in [K]$ and $i \in [N]$, it holds that $\sum_{i=1}^N v_{i,k} = 1$. Let $\sigma \in [K]^N$ be a labeling. Then, the $K$-ary cross-entropy loss (or, log-loss) on $\mathbf{v}$ with respect to $\sigma$ is defined as follows:
% $\logloss\paran{\mathbf{v}; \sigma} := \frac{-1}{N}\sum_{i=1}^N \sum_{k=1}^K \ \Big([\sigma_i = k] \cdot \ln v_{i,k} \Big)$,   
% where $[\sigma_i = k] = 1$ if $\sigma_i = k$ and $0$, otherwise\footnote{Note that the expression reduces to Def.~\ref{def:log_loss_binary} for $K=2$.}. 
We prove that using the powers of a distinct prime for each datapoint, we can infer all the labels in a single query -- in particular, using vector $\mathbf{v}_K$ of the following form:
% \vspace{-0.em}
$$v_{i,k} = \frac{p_i^{k-1}}{\sum_{j=1}^K p_i^{j-1}}\ \ \forall (i, k) \in [N] \times [K].$$ 

The proof of correctness follows from the Fundamental Theorem of Arithmetic (see Appendix~\ref{app:multi} for details).

\begin{theorem}\label{thm:multi_class_log_loss_attack}
There exists a polynomial-time adversary for K-ary label inference in the $\mathsf{APA}$ model using only a single log-loss query. For inference in the $\mathsf{FPA}(\phi)$ model, it suffices to issue $O\paran{1 + NK h(\phi)}$ queries, where the following holds when $K < N$: $$h(\phi) = O\paran{\frac{(\ln \phi)^2}{\paran{\phi + (N-K)\ln K}^{2/3}}}.$$
% \shiva{is there an expression for $h(\phi)$?}
\end{theorem}
Observe that $\lim_{\phi \to \infty} h(\phi) = 0$, as expected.

\section{Label Inference from Noised Scores}
\label{sec:bounded_noise}
In this section, we describe a label inference attack for binary labels that works even when the reported scores are noised before the adversary gets to see them. We do not place any assumption on the noise distribution except that the adversary knows an upper bound on the amount of resulting error. Compared to attacks in the unnoised case presented in the previous section, our attacks here use a larger number of queries (for the same number of bits of precision) and currently only work for the binary label case. 

% \footnote{We reiterate that while the discussion in this section will also hold for the unnoised case (e.g., by using $\tau=1$), the latter uses a fewer number of bits of precision and can be easily extended to the multi-class setting. Extension of our attacks in the noised case to the multi-class setting is an interesting future direction.}.

We begin by formally defining the label inference problem in presence of noise. Missing details from this section are collected in the Appendix~\ref{app:bounded_noise}.
% \shiva{explain again why Section~3 results are still interesting?}
\begin{definition}
    \label{def:label_inference_nois}
    Let $\tau > 0$ and $\sigma \in \{0,1\}^N$ be the (unknown) labeling. The $\tau$-robust label inference problem is that of recovering $\sigma$ given $\ell_1,\dots,\ell_M$, where 
    for all $i \in [M]$, it holds that $|\logloss\paran{\mathbf{u}_i; \sigma} - \ell_i | \leq \tau$. Here, $M$ is the number of queries and  $\mathbf{u}_i \in [0,1]^N$ are the prediction vectors.  
\end{definition}
As before, we discuss the results in both  $\mathsf{APA}$ and $\mathsf{FPA}$ arithmetic. In this case, we also make a distinction between exponential- and polynomial-time adversaries. 

\subsection{$\tau$-Robust Label Inference under Arbitrary Precision with Exponential-time Adversary} 
As in Section~\ref{sec:rawScores},  because of the equivalence between $\mathcal{L}_{\mathbf{v}}(\sigma)$ and $\logloss\paran{\mathbf{u};\sigma}$ for $\mathbf{u} = f(\mathbf{v})$, we focus on recovering $\sigma$ from $\mathcal{L}_{\mathbf{v}}(\sigma)$.  We start with a  definition of a vector $\mathbf{v}$ that helps with label inference in the presence of noise. To illustrate our key ideas, we first focus on an exponential-time adversary, and later extend the results to a polynomial-time adversary.

% We emphasize that this is a stronger privacy definition than that given for blatant non-privacy in~\cite{dinur2003revealing}. Moreover, handling this case allows inference from other noise distributions as well.

\begin{definition}\label{def:robustness}
Let $\tau > 0$. In the $\mathsf{APA}$ model, we say that a vector $\mathbf{v} \in (0,\infty)^N$ is $\tau$-robust if for all labelings $\sigma \in \{0,1\}^N$ and all $\ell$ such that $\abs{\mathcal{L}_{\mathbf{v}}\paran{\sigma} - \ell} \leq \tau$, there exists an algorithm (Turing Machine) $\mathcal{A}$ such that $\mathcal{A}\paran{\ell, N, \tau, \mathbf{v}} = \sigma$.
\end{definition}

% \begin{definition}\label{def:robustness}
% Let $\phi > 0$ be an integer and $\tau > 0$. We say that a vector $\mathbf{v} \in (0,\infty)^N$ is $(\tau, \phi)$-robust if for all $\sigma \in \{0,1\}^N$ and all $\ell$ such that $\abs{\mathcal{L}_{\mathbf{v}}\paran{\sigma} - \ell} \leq \tau$, there exists an algorithm (Turing Machine) $\mathcal{A}$ in the $\mathsf{FPA}(\phi)$ model that can recover $\sigma$ from $\ell$, \emph{i.e.} $\mathcal{A}\paran{\ell, N, \tau, \mathbf{v}} = \sigma$. 

% In the $\mathsf{APA}$ model, we say that $\mathbf{v} \in (0,\infty)^N$ is $\tau$-robust if exists an algorithm (Turing Machine) $\mathcal{A}$ such that $\mathcal{A}\paran{\ell, N, \tau, \mathbf{v}} = \sigma$.
% \end{definition}

We show that for all $\tau > 0$, there exists a $\tau$-robust vector in the $\mathsf{APA}$ model that can recover the true labeling in a single query. We do so by constructing a vector $\mathbf{v}$ such that for any two different labelings $\sigma_1, \sigma_2 \in \{0,1\}^N$, it holds that
$| \mathcal{L}_\mathbf{v}(\sigma_1) - \mathcal{L}_\mathbf{v}(\sigma_2) | > 2\tau$, so that the inference from the noised loss is unambiguous. To achieve this co-domain separation, we reduce our problem to constructing sets with a given minimum difference between its arbitrary subset sums. The construction from this reduction will help design our vector $\mathbf{v}$. The main steps of our approach are outlined in Algorithm~\ref{alg:bounded_noise_exponential} and described below. 

Let $\Delta(\mathbf{v}) := \min_{\sigma_1, \sigma_2 \in \{0,1\}^N} \left| \mathcal{L}_{\mathbf{v}}\paran{\sigma_1} - \mathcal{L}_{\mathbf{v}}\paran{\sigma_2} \right|$ be the magnitude of the minimum difference in the loss scores computed on any two distinct labelings. For any set $S$, let $\mu(S) := \min_{S_1, S_2 \subseteq S} \left| \sum_{s_1 \in S_1} s_1 - \sum_{s_2 \in S_2} s_2 \right|$ denote the magnitude of the minimum difference between any two subset sums in $S$. Remember that our goal is to ensure that $\Delta(\mathbf{v})$ is large (in particular, more than $2\tau$). The following lemma will be helpful in constructing such a vector $\mathbf{v}$.

\begin{lemma}\label{lem:leaky_sensitivity_helper}
	Let $\mathbf{v} = [v_1,\dots,v_N]$ be a vector with all entries distinct and positive. Define $\ln \mathbf{v} := [\ln v_1,\dots,\ln v_N]$. Then, it holds that $\Delta(\mathbf{v}) = \frac{1}{N}\mu(\ln \mathbf{v})$.
\end{lemma}

Now, let $\mathcal{S}_N = \{1,2^1,\dots,2^{N-1}\}$. Observe that $\mu\paran{\mathcal{S}_N} = 1$. This follows from the fact that $\mathcal{S}_N$ contains all $N$ distinct powers of $2$ (from 0 to $N-1$), and hence, every subset of $\mathcal{S}_N$ has a distinct sum since there are $2^N$ subsets of $\mathcal{S}_N$ and each subset corresponds to a unique integer in $[2^N-1]$. Using this set, we can achieve the desired separation between loss scores by scaling the elements in $\mathcal{S}_N$ as suggested by Lemma~\ref{lem:leaky_sensitivity_helper} and the construction in Algorithm~\ref{alg:bounded_noise_exponential}. We prove the correctness of this approach in the following.

\begin{algorithm}[t]
\caption{Label Inference with Bounded Error in the $\mathsf{APA}$ Model (Exponential Adversary)}
 \label{alg:bounded_noise_exponential}
 \begin{algorithmic}[1]
  \STATE \textbf{Input:} $N$, upper bound on error $\tau > 0$
  \STATE \textbf{Output:} Labeling $\hat{\sigma} \in \{0,1\}^N$
  \STATE Set $\mathbf{u} \gets f(\mathbf{v})$, where $v_i \gets 3^{2^iN\tau}$.
  \STATE Obtain the loss score $\ell$ using $\mathbf{u}$ as the prediction vector.
  \STATE \textbf{Return} $\hat{\sigma} \gets \arg\min_{\sigma \in \{0,1\}^N} \abs{\mathcal{L}_{\mathbf{u}}\paran{\sigma} - \ell}$.
 \end{algorithmic}
\end{algorithm}
% \shiva{Change this theorem from $\tau$-robust vector to a claim on $\tau$-robust label inference problem as in previous theorems.} \abhi{Done. Please check.}
\begin{theorem}\label{thm:tau_leaky}
For any $\tau > 0$, there exists an exponential-time adversary (from Algorithm~\ref{alg:bounded_noise_exponential}) for single-query $\tau$-robust label inference in the $\mathsf{APA}$ model.
% requiring only a single log-loss query.
\end{theorem}
% \begin{proof}
% Let $S = \{s_1,\dots,s_N\}$ be a set such that $\mu(S) \geq 2N\tau$. We know that $S$ exists, for example, by scaling each element of $\mathcal{S}_\circ$ by $2N\tau$. Let $\mathbf{v}$ be constructed such that $v_i = 3\exp(s_i)$. To see that $\mathbf{v}$ is $\tau$-robust, it suffices to prove that $\Delta(\mathbf{v})> 2\tau$. Now, observe that $(\ln 3)s_i = \ln \mathbf{v}_i$. From Lemma~\ref{lem:leaky_sensitivity_helper}, we get  $\Delta(\mathbf{v}) = \frac{\ln 3}{N}\mu(S) = (2\ln 3)\tau > 2\tau$. 
% \end{proof}
% \vspace{-1em}
We note a few keys remarks about Algorithm~\ref{alg:bounded_noise_exponential}. First, note that the exact knowledge of $\tau$ is not required -- any upper bound $\tau_{\text{max}}$ suffices. This eliminates the need for the adversary to know the exact noise generation process and compute the bound $\tau_{\text{max}}$ purely from its knowledge of the environment (\emph{e.g.} precision on the channel through which the scores are communicated). Second, at first glance, it may seem too good to be true that we can handle arbitrary noise levels added to the noise scores. Intuitively, too much noise must render any signal completely useless when recovering meaningful information. However, we remind the reader that Algorithm~\ref{alg:bounded_noise_exponential} requires arbitrary precision. In practice, any finite precision will limit the resolution to which the difference between the loss scores can be controlled using vectors that contain entries that are exponentially large in $N$ and $\tau$. As we will show next, multiple queries are required in this case, which, instead of separating scores over all the labelings at once, only separates scores over labelings that differ in a few bits (which can be significantly smaller in number). Third, note that the adversary iterates over the entire exponentially large ($2^N$) space of the labelings to recover the true labeling $\sigma$ (see step 5). While this is feasible if we assume an all-powerful adversary, in more realistic scenarios where the adversary is limited to only polynomial computations, at most $O(\log N)$ bits must be inferred at a time. We discuss this approach in more detail below. Lastly, observe that Algorithm~\ref{alg:bounded_noise_exponential} is a single-query inference. Even with the caveats mentioned above, this algorithm is a certificate to the guarantee that even noisy loss scores leak sufficient information about the private labels, which, given enough computation power, can be extracted unambiguously. A natural question to ask is if it is possible to perform this inference using smaller numbers than what our construction uses. We now explore if this is possible.

\textbf{Optimality of Single-Query $\tau$-Robust Inference.} \label{sec:opti}
% \shiva{My suggestion will be to shrink this section, have 1-2 paras, and the result}
% Recall that according to Definition~\ref{def:label_inference_nois}, we must construct a vector $\mathbf{v}$ that, for any labeling $\sigma$, can recover $\sigma$ from a noised log-loss score computed on this labeling. This is a much stronger definition than  requiring the recovery of only the true labeling. Our construction in Algorithm~\ref{alg:bounded_noise_exponential} set the elements of $\mathbf{v}$ that were double-exponentially large in the number of elements $N$ and exponential in the noise bound $\tau$, which requires a large number of bits of floating point precision to be simulated in practice. In what follows, we prove that this large size of elements in $\mathbf{v}$ is asymptotically optimal. In other words, 
We now prove that any solution to the single-query Robust Label Inference Problem must use a prediction vector that has large entries. At a high level, we do this by first reducing the problem of constructing sets with distinct subset sums into the problem of constructing $\tau$-robust vectors\footnote{Recall we used the reverse direction earlier in our construction.}. Having established the equivalence between the two problems through this reduction, we then prove our lower bound by generalizing a classic result by Euler, which states that any set of positive integers with distinct subset sums must have at least one element that is exponentially large in the number of elements in the set~\cite{benkoski1974weird,frenkel1998integer}. We state this result as a theorem since it may be of independent mathematical interest (missing details in Appendix~\ref{app:opti}). We let $\mathbb{Q}^{+}$ denote the set of positive rational numbers and for any set $S$, let $\pnorm{\infty}{S} = \max_{s \in S}\abs{s}$.

\begin{theorem}\label{thm:euler_general}
   For any set $S \subset \mathbb{Q}^{+}$ with $\mu(S) > \lambda$ for some $\lambda \in [0,\infty)$, it must hold that $\pnorm{\infty}{S} = \Omega(\lambda 2^{|S|})$.
\end{theorem}
The optimality of our construction of prediction vectors can now be proven as follows.

\begin{theorem}\label{thm:lower_bound_helper}
For sufficiently large $N$ and all $\tau > 0$, any $\tau$-robust vector $\mathbf{v}$ must have $\pnorm{\infty}{\mathbf{v}} = \Omega\paran{e^{2^N N \tau}}$. 
\end{theorem}
We emphasize that this lower bound is only when a single log-loss query is used for inference (like used in Algorithm~\ref{alg:bounded_noise_exponential}). This is because if multiple queries can be issued, then smaller numbers in the vector construction may suffice. For example, using $N$ queries with vectors of the form $\mathbf{v} = [k,1,\dots,1]$, where $k > e^{2N\tau}$ suffice. 
% In each query, the number of bits required to represent $\mathbf{v}$ and the loss scores is $O(N\tau)$, which is much less than the bound in Theorem~\ref{thm:lower_bound_helper}. 
% What is common, though, is the linear dependence on $\tau$, which we leave as an interesting open problem to check for tightness.

\subsection{$\tau$-Robust Label Inference under Bounded Precision with Polynomial-time Adversary} 
We now present our algorithm for label inference under bounded number of bits of precision. As discussed above, at most $O(\log N)$ bits must be inferred at a time since any larger amount will be infeasible in polynomial time. The idea behind our construction is similar to that in Algorithm~\ref{alg:multi_query}, only this time, we construct prediction vectors that offer robustness to the noise added. We outline the main steps in Algorithm~\ref{alg:reconstruction_attack} and discuss it below. Note that Algorithm~\ref{alg:reconstruction_attack} only requires an upper bound on $\tau$.

\begin{algorithm}[t]
\caption{Label Inference with Bounded Error in the $\mathsf{FPA}(\phi)$ Model (Polynomial Adversary)}
 \label{alg:reconstruction_attack}
 \begin{algorithmic}[1]
  \STATE \textbf{Input:} $N$, upper bound on error $\tau > 0$, $\phi$
  \STATE \textbf{Output:} Labeling $\hat{\sigma} \in \{0,1\}^N$
  \STATE Initialize $\hat{\sigma} \gets [0,\dots,0]$.
  \STATE Let $m \gets \min \left\{ \left\lceil \log_2 N \right\rceil, \left\lfloor \log_2 \parfrac{\phi -8}{N \tau \ln 2} \right\rfloor \right\}$. 
  \FOR{$i$ \textbf{in} $\left\{1,\dots,\left\lceil \frac{N}{m} \right\rceil \right\}$}
    \STATE Form a vector $\mathbf{v}$ with $\mathbf{v}_j = 3^{2^{r}N\tau}$ if $j =
                (i-1)m + r$ for some $r\in[m]$. Else, set $\mathbf{v}_j = 1$.
    \STATE Let $\mathbf{u} = f(\mathbf{v}) = \left[ \frac{v_1}{1+v_1}, \dots, \frac{v_N}{1+v_N}\right]$. Obtain the loss score $\ell$ using $\mathbf{u}$ as the prediction vector.
    \STATE Let $\Sigma_{w,m} = 0^{(w-1)m}(0+1)^m0^{N-wm}$ (expressed as a regular expression -- truncated to $N$ bits) denote the set of all labelings that have $0$'s in the first $(w-1)m$ and last $(N-wm)$ positions.
    \STATE Compute $\sigma' \gets \arg\min_{\sigma \in \Sigma_{i,m}} \abs{\mathcal{L}_{\mathbf{u}}\paran{\sigma} - \ell}$.
    \STATE Set $\hat{\sigma}_k = \sigma'_k$ for $k \in \{(i-1)m+1,\dots,im\}$, while keeping all other entries in $\hat{\sigma}$ unchanged.
  \ENDFOR
  \STATE \textbf{Return} $\hat{\sigma}$.
 \end{algorithmic}
\end{algorithm}

\begin{lemma}\label{cor:vec_attack}
	Let $\tau > 0$ be a bound on the resulting error and $m \leq N$ be an integer. Let $\mathbf{v}_m = \left[ 3e^{2N\tau}, 3e^{4N\tau}, \dots, 3e^{2^{m}N\tau}, 1, \dots, 1 \right]$.
	Then, for any distinct $\sigma_1, \sigma_2 \in \{0,1\}^N$, it holds that if $\sigma_1[:m] = \sigma_2[:m]$, then $\mathcal{L}_{\mathbf{v}_m}\paran{\sigma_1} = \mathcal{L}_{\mathbf{v}_m}\paran{\sigma_2}$. Else, we have $\left| \mathcal{L}_{\mathbf{v}_m}\paran{\sigma_1} - \mathcal{L}_{\mathbf{v}_m}\paran{\sigma_2} \right| > 2\tau$.
\end{lemma}

If using the construction vector $\mathbf{v}_m$ from this above lemma, for any $m \leq N$, the loss scores computed on vector $\mathbf{v}_m$ are at least $2\tau$ apart for all labelings that differ in at least one index in $[m]$. Moreover, if the first $m$ bits are the same for any two labelings, this lemma tells us that the loss scores will be the same as well. This is the key idea that allows unambiguous label inference in Algorithm~\ref{alg:reconstruction_attack}. We formally state our result about the correctness of this algorithm and compute a bound on the number of bits required as follows.

% \shiva{Say $\phi > N \tau$ of $> N \tau_{\text{max}}$, here and throughout the paper. Explain why this is necessary.}
% \shiva{should $\tau$ be replaced by $\tau_{\text{max}}$?}
\begin{theorem}\label{thm:proof_reconstruction}
	For any error bounded by $\tau > 0$ and $\phi \geq 8 + \lceil N\tau \ln 2 \rceil$, there exists a polynomial-time adversary (from Algorithm~\ref{alg:reconstruction_attack}) for the $\tau$-label inference problem in the $\mathsf{FPA}(\phi)$ model using $O\paran{\frac{N}{\log N} + \frac{N}{\log \paran{\phi/N\tau}}}$ queries.
\end{theorem}
% \begin{proof}
% 	It suffices to show that in the $i^{th}$ iteration of the for-loop in Algorithm~\ref{alg:reconstruction_attack}, the vector $\sigma'$ contains the true labels for indices in $\{(i-1)M+1,\dots,iM\}$.  Without loss of generality, assume $i=1$, so that the vector $\mathbf{v} = \left[ 3e^{2N\tau}, 3e^{4N\tau}, \dots, 3\exp\paran{2^{M}N\tau}, 1, \dots, 1 \right]$.
%     From Lemma~\ref{cor:vec_attack}(b), it follows that for any distinct $\sigma_1, \sigma_2 \in \{0,1\}^N$ where $\sigma_1[:M] \neq \sigma_2[:M]$, it holds that:
%     \begin{align}
%         \label{eq:bound_on_loss}
%         \left| \mathcal{L}_{\mathbf{v}}\paran{\sigma_1} - \mathcal{L}_{\mathbf{v}}\paran{\sigma_2} \right| > 2\tau.
%     \end{align}
%     Thus, when the loss $\ell$ is observed on $\mathbf{u} = f(\mathbf{v})$ and $\arg\min_\sigma \abs{\mathcal{L}_{\mathbf{u}}\paran{\sigma} - \ell} = \{\sigma^{(1)},\dots,\sigma^{(k)}\}$, then it must be true that $\sigma^{(k_1)}[:M] = \sigma^{(k_1)}[:M]$ for all $k_1,k_2 \in [k]$ (or else, it would contradict the inequality in~(\ref{eq:bound_on_loss})). Furthermore, from Lemma~\ref{cor:vec_attack}(a), it follows that the true labeling must be present in the set $\{\sigma^{(1)},\dots,\sigma^{(k)}\}$. Thus, at the end of iteration $i=1$, the vector $\hat{\sigma}$ has recovered the first $M$ bits in $\sigma$. For all other iterations, note that the vector $\mathbf{v}$ is just a cyclic rotation to the right by $M$ elements, and hence, the bits in $\sigma$ are recovered, $M$ at a time, in Algorithm~\ref{alg:reconstruction_attack}. 
% \end{proof}
\subsection{Extension to Other Noise Models}\label{sec:extension_other_noise}
In the previous section, we considered the case where the log-loss responses were all accurate within $\pm \tau$ for some $\tau > 0$. We now give an overview of how our attacks could be  presence of random and multiplicative noise (missing details in Appendices~\ref{app:random_noise} and~\ref{app:multiplicative_noise}, respectively).

\noindent\textbf{Random Noise.} 
The results above on bounded error can be extended  to handle randomly generated (additive) noise. 
We achieve this by safeguarding against the worst case magnitude of the noise that can be added for bounded noise distributions. For cases where this distribution is unbounded, we allow for some error tolerance. 
% In particular, we prove that Robust Label Inference can be solved within a single query for all distributions $\mathcal{D}$ that have a bounded support over $\mathbb{R}$. If $\mathcal{D}$ has unbounded support, this inference can be performed with probability at least $1-\delta$ for any desired $\delta \in (0,1)$ (the number of bits required increases when error tolerance is low). 

To see why this works, observe that for any bounded distribution, say over some interval $[a,b] \subset \mathbb{R}$, the error is bounded by $\max\{|a|,|b|\}$. Thus, using any $\max\{|a|,|b|\}$-robust vector will suffice. For distributions with unbounded support, however, this upper bound does not exist. However, given a failure probability $\delta > 0$, it is possible to compute this bound for vector robustness for any distribution. For example, in the $\mathsf{APA}$ model, in case of subexponential noise\footnote{A random variable $X$ is subexponential with parameters $\lambda^2$ and $\nu$ if 
% the moment generating function satisfies: 
$\mathbb{E}(e^{sX}) \leq \exp{\paran{\lambda^2s^2/2}}$ for all $|s| < 1/\nu$} with parameters $\lambda^2$ and $\nu>0$, and $\delta \in (0,1)$, a $\tau$-robust vector $\mathbf{v}$ constructed as in Algorithm~\ref{alg:bounded_noise_exponential} (with $\tau = \paran{2(\lambda+\nu)\sqrt{\ln 2/\delta}}$) will succeed in label inference with probability at least $1-\delta$.
\begin{figure*}[t]
    \centering
    % \textbf{\small Plots for Label Inference} \abhi{TODO for myself: Add y-axis and remove N/10 plot.}
    \vspace{-0.5em}
    \subfloat[]{\includegraphics[width=120pt]{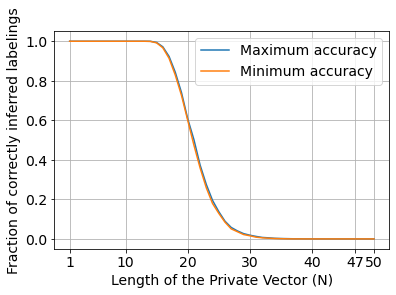}}
    \hspace{-0.1em}
    % \subfloat{\includegraphics[width=150pt]{icml21/hits_unnoised.png}}
    % \hspace{-0.1em}
    \subfloat[]{\includegraphics[width=120pt]{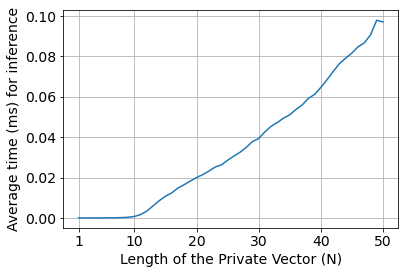}} 
    % \newline
    % \textbf{\small Plots for Multi-Query Inference with No Noise}
    \hspace{-0.1em}
    % \subfloat{\includegraphics[width=150pt]{icml21/multi_accuracies_unnoised.png}}
    % \hspace{-0.1em}
    \subfloat[]{\includegraphics[width=120pt]{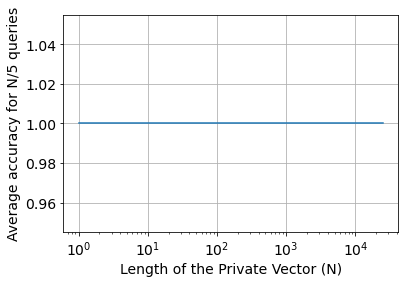}}
    \subfloat[]{\includegraphics[width=120pt]{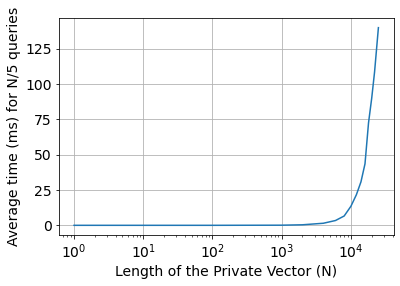}}
    \vspace{-1em}
    \newline
    % \textbf{\small Plots for Robust Label Inference}\\
    \subfloat[]{\includegraphics[width=120pt]{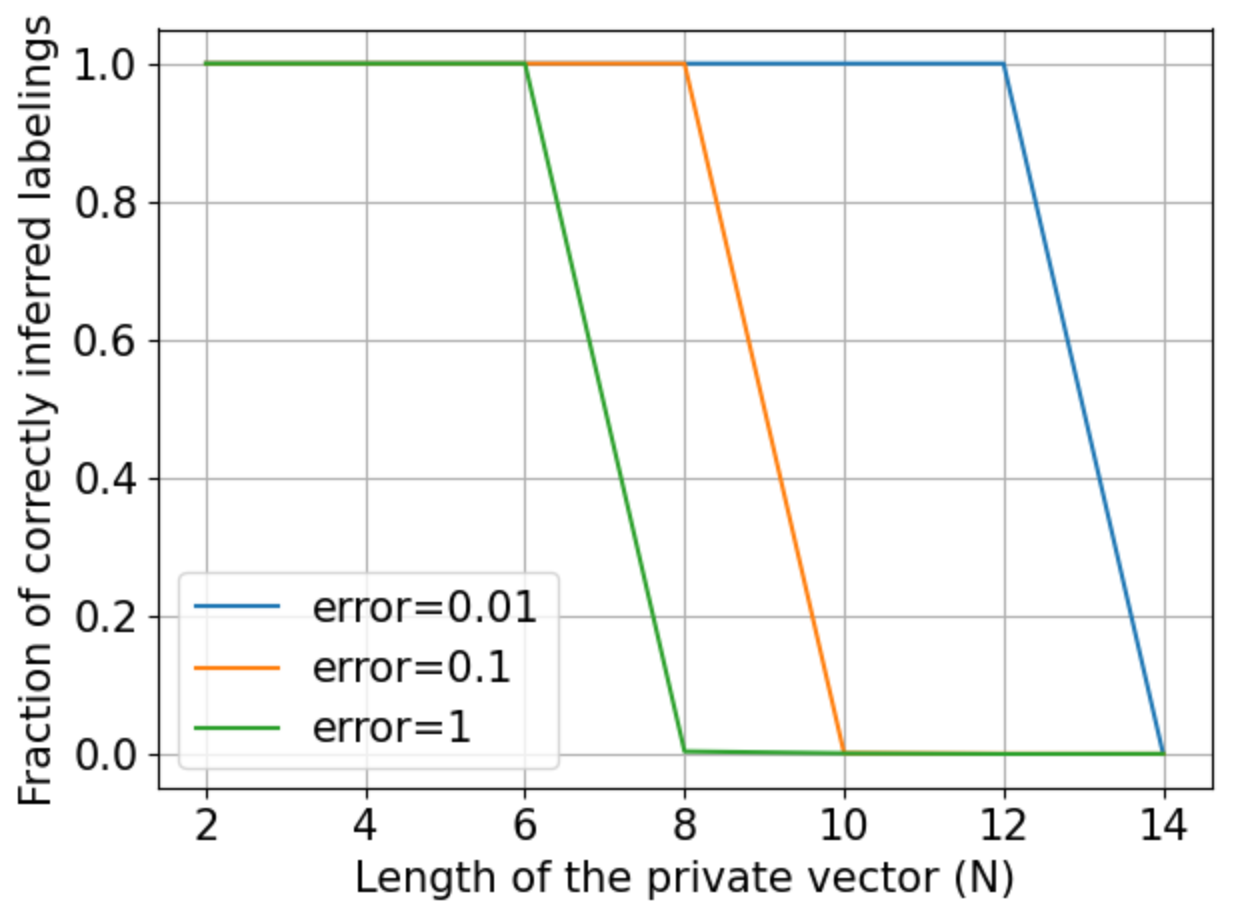}}
    \hspace{-0.1em}
    % \subfloat{\includegraphics[width=150pt]{icml21/noisy_oneshot_num_acc.png}}
    \subfloat[]{\includegraphics[width=120pt]{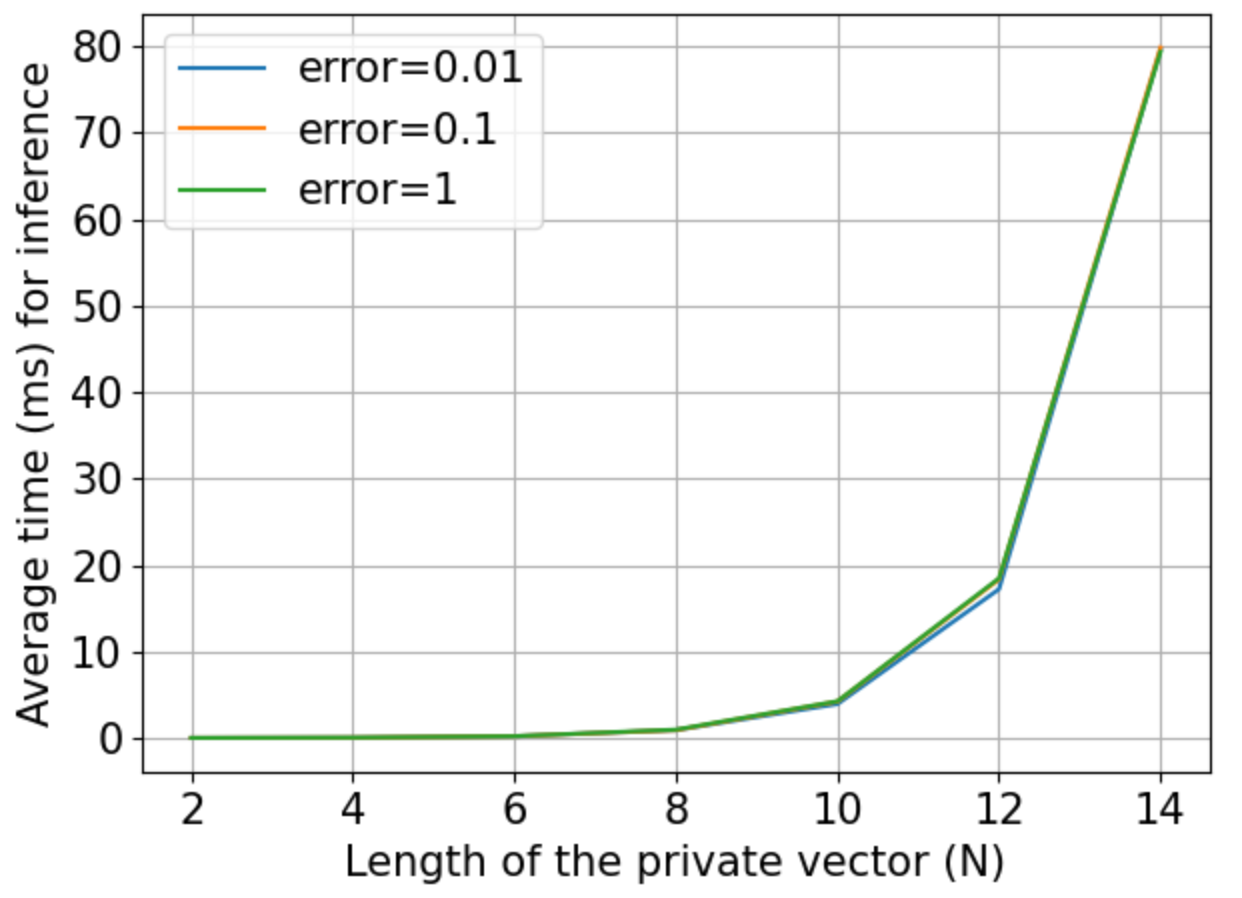}}
    \hspace{-0.1em}
    % \newline
    % \textbf{\small Plots for Multi-Query Inference with Additive Bounded Noise}\\
    \subfloat[]{\includegraphics[width=120pt]{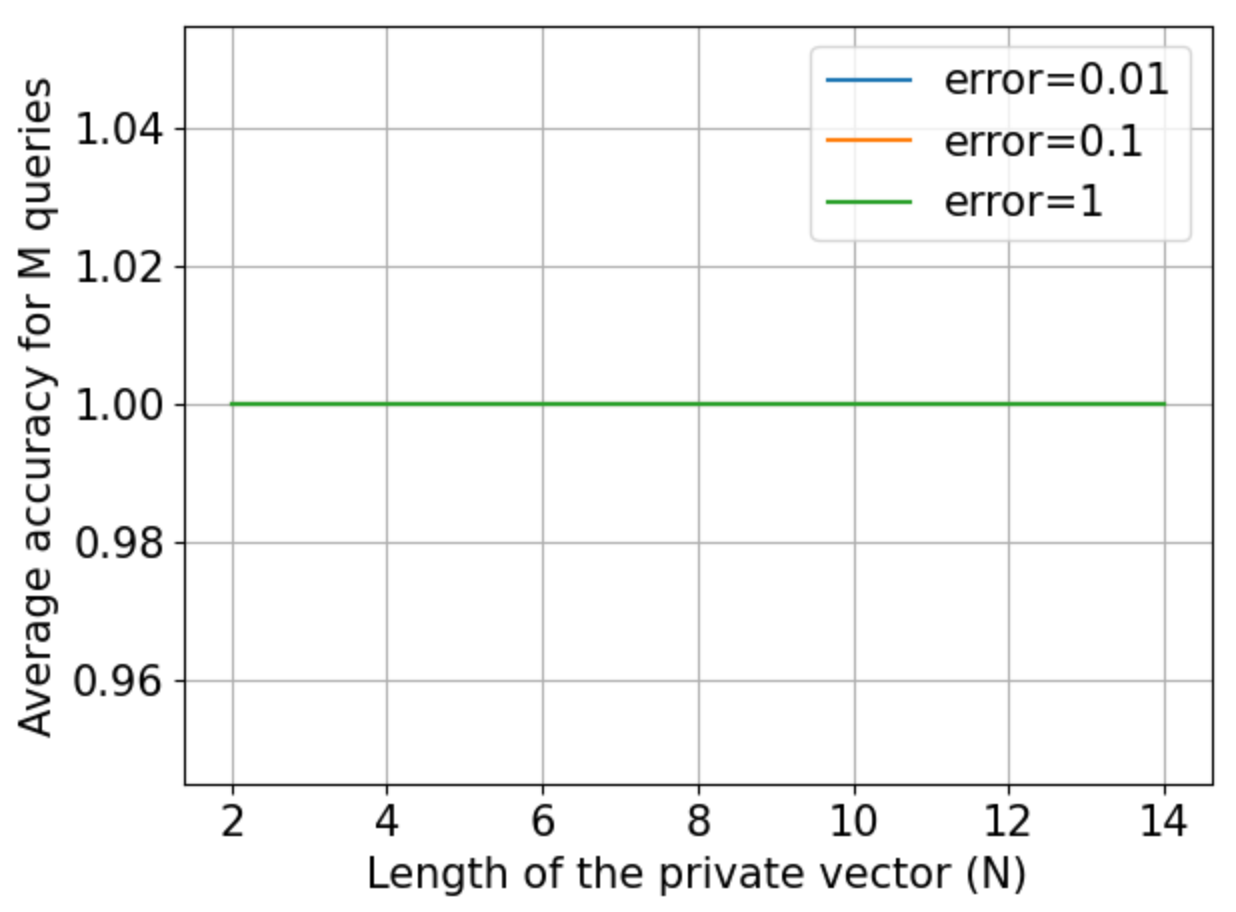}}
    \hspace{-0.1em}
    % \subfloat{\includegraphics[width=150pt]{icml21/noisy_multishot_accnum_phi51015.png}}
    \subfloat[]{\includegraphics[width=120pt]{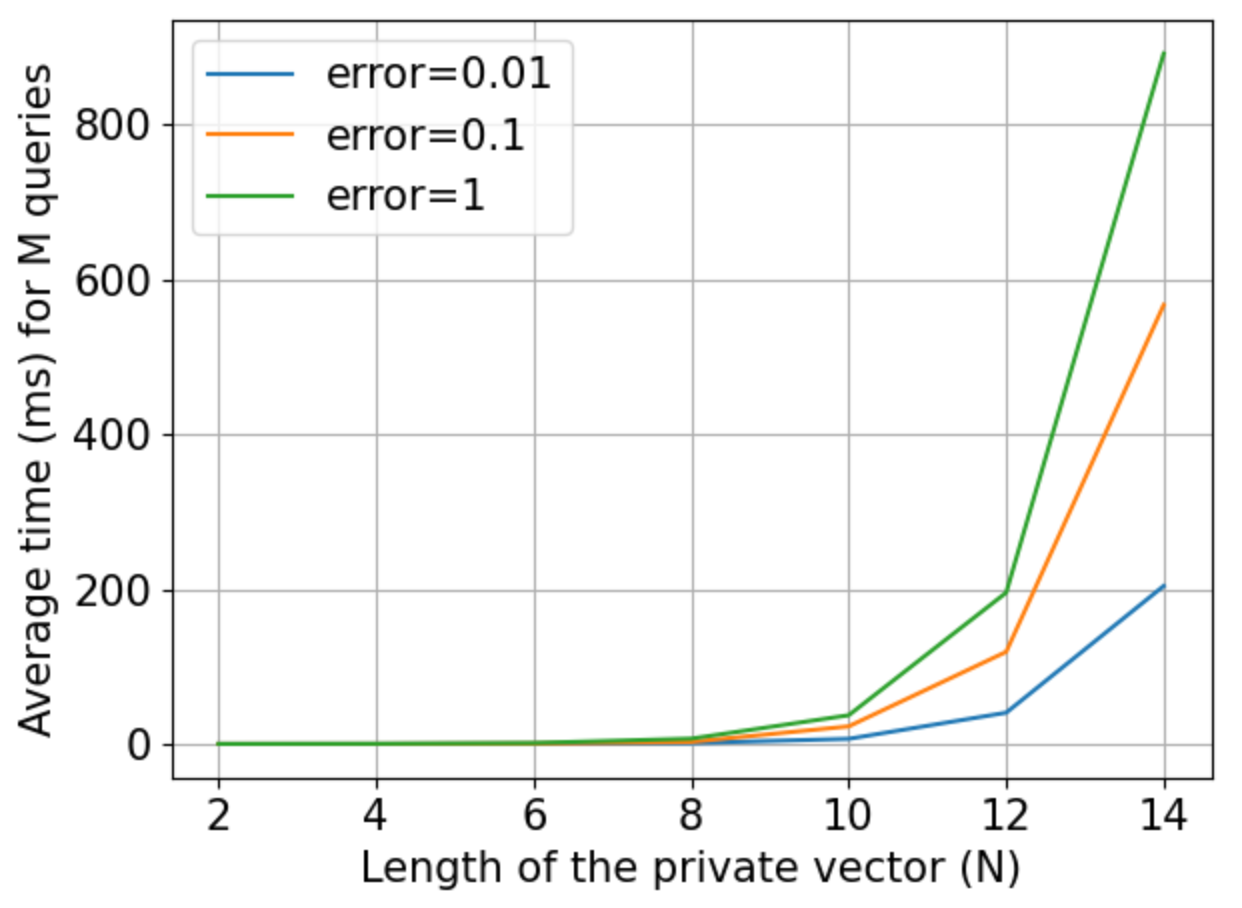}}
    \vspace*{-1ex}
    \caption{Results on simulated binary labelings. The first row shows the performance of the label inference without noise using for single-query (a) and (b), and multi-query (c) and (d). The second row shows the performance of label inference with bounded error (scale = $0.01$, $0.1$, and $1$) for single-query (e) and (f), and multi-query (g) and (h). Here, $M$ is the number of queries from Algorithm~\ref{alg:reconstruction_attack}.}
    \label{fig:experiment_plots}
    \vspace*{-2ex}
\end{figure*}

\noindent\textbf{Multiplicative Noise.} We briefly explore extending the analysis above to the case of multiplicative error. In this case, the adversary observes score $\ell$ that satisfies $(1-\alpha)\mathcal{L}_\mathbf{v}(\sigma) \le \ell \le (1+\alpha)\mathcal{L}_\mathbf{v}(\sigma)$, where the bound $\alpha \in [0,1)$ is known to the adversary. 
% bound on the rate of the multiplicative noise. 
% We begin with rewriting the inequality above as  $\abs{\ell - \mathcal{L}_\mathbf{v}(\sigma)} \leq \alpha \cdot \abs{\mathcal{L}_\mathbf{v}(\sigma)}$. 
We prove that for any $\alpha \leq 1/8$, label inference can be done, $\left\lceil \log_2 \parfrac{1}{\alpha} - 2 \right\rceil$ labels at a time, using vectors that are $(2\ln 2)\alpha$-robust. Observe that when $\alpha \geq 1/4$, then no value of $\tau$ satisfies the constraint above, implying that vectors from Algorithm~\ref{alg:reconstruction_attack} cannot be used with any number of queries. The noise is more than what these vectors can guarantee handling.

\section{Experimental Observations}
\label{sec:expts}
We evaluate our attacks on both simulated binary labelings and real binary classification datasets fetched from the UCI machine learning dataset repository\footnote{https://archive.ics.uci.edu/ml/machine-learning-databases}. In this section, we focus on the binary label experiments, deferring the multiclass experiments to Appendix~\ref{app:experiments_multi_class}.  The results show that our algorithms are surprisingly efficient, even with a large number of datapoints. 
% We first evaluate the performance of Algorithm 1 for single-query and multi-query label inference with no noise. In the bounded noise setting, Algorithm 2 is used for the single-query label inference while Algorithm 3 is used for multi-query inference. 
All experiments are run on a 64-bit machine with 2.6GHz 6-Core processor, using the standard IEEE-754 double precision format (1 bit for sign, 11 bits for exponent, and 53 bits for mantissa). For ensuring reproducibility, the entire experiment setup is submitted as part of the supplementary material.

\textbf{Results on Simulated Binary Labelings.} The first row in Figure~\ref{fig:experiment_plots} shows the plots for label inference with no noise, where we use the attack based on primes from Theorem~\ref{thm:basic_log_loss_attack}. The accuracy reported is with respect to 10000 randomly generated binary labelings for each $N$ (length of the vector to be inferred). For $N \le 10$ all labels are correctly recovered (see Figure~\ref{fig:experiment_plots}(a)). For $N \ge 47$, the maximum accuracy falls to zero. This is not unexpected because of the limited floating point precision on the machine.
%\shiva{explain why this happens?}
% which is due to the floating point precision on the machine. 
%The plot in the middle shows the number of labelings (out of the $2^N$ total) that were correctly inferred in a single run of the attack, which indicates that even if precision is a problem, there can be some labelings that can be correctly inferred.
%\shiva{This plot is misleading. Change y-axis to fraction of recovered, rather than absolute.}\zekun{reply: Working to remove it. If showing fraction, it will be the same as the 1st plot.} 
The run time plot shows that this inference happens in only a few milliseconds (see Figure~\ref{fig:experiment_plots}(b)).  
%\shiva{and scales linearly with $N$?}\zekun{Added text: which scales linearly with $N$.}
%\shiva{Motivate the need for multiquery.}
% Since the single-query label inference performance suffers from the floating point precision on the machine, we want to exploit the label inference with multiple queries so as to increase the accuracy.
The third plot in row 1 shows  multi-query label inference with no noise using Algorithm~\ref{alg:multi_query}. We use $N/5$ queries.\footnote{We did not optimize for the number of queries -- probably a smaller number of queries suffice for 100\% reconstruction.}
The results show that by changing the number of queries from 1 to $N/5$, we now have accuracy of 100\% even up to $N=10,000$, with the corresponding average run time shown in the fourth plot (Figure ~\ref{fig:experiment_plots}(d)). The average runtime is order of few milliseconds, even when $N=10000$, demonstrating the efficiency of this attack.
%\shiva{why  this difference between N/5 and N/10?}\zekun{reply: $N/10$ is still not a large enough number of queries so the precision issue persists. working to remove the $N/10$ and only show $N/5$ plot, so each row will have 2 plots.

The second row in Figure~\ref{fig:experiment_plots} shows the plots for robust label inference with bounded noise. 
%\shiva{Explain this setup, is it similar to first row, what algorithm is used for inference.}
The setup is similar to the first row, except that a bounded noise of scale $0.01$, $0.1$, or $1$ is added to the log-loss scores (the noise is from about 1\% to 100\% 
of the raw log-loss score).
In Figure~\ref{fig:experiment_plots}(e), the label inference is performed using Algorithm~\ref{alg:bounded_noise_exponential}. 
As the noise scale increases from $0.01$ to $0.1$ and $1$, the length of the private vector on which we can recover all labelings correctly drops from 12 to 8 and 6, respectively. 
%The middle plot shows how the number of correctly inferred labelings as a function of the length of the private vector for different noise scales.
%\shiva{Again change middle plot to fraction}\zekun{working to remove it} 
The run time displayed in the Figure~\ref{fig:experiment_plots}(f) is much higher than in the unnoised case (Figure~\ref{fig:experiment_plots}(b)) because the label inference algorithm for the noised setting  involves iterating through $2^N$ labelings to find out the one that is closest to the reported noisy log-loss.
%\shiva{Are there 3 lines in the last plot?}\zekun{reply: yes, they largely overlap with each other}
%\shiva{Again are there 3 lines in this plot?}
%\zekun{reply: Yes, they are on top of each other.}\shiva{It will be good to have row 4 have the same form as plots in row 2 with N/10 and N/4?}\zekun{reply: this may not be doable because in Algorithm 3, the number queries M is a function of the precision, noise, and sample size, which cannot be arbitrarily specified.}
% Since the accuracy of single-query label inference decreases with $N$ as the noise scale increases, we once again utilize multi-query inference technique to improve the accuracy. 
The third plot shows the accuracy for multi-query label inference with bounded noise using Algorithm~\ref{alg:reconstruction_attack}. With multiple queries calibrated by the noise scale, we are able to recover all labelings correctly in all three noise cases,
% (three lines on top of each at 100\%), 
with the corresponding run time displayed in Figure~\ref{fig:experiment_plots}(f).
\begin{table}[t]
    \centering
    \caption{Experimental results on real datasets using Algorithm~\ref{alg:multi_query}. Here, $N$ is the number of test samples in the dataset and \textbf{Acc$_q$} is the fraction of labels correctly inferred with $q$ queries.}\vspace{0.1in}
    \begin{tabular}{|c|c|c|c|c|}
\hline
     \textbf{Dataset}& \textbf{N}& \textbf{Acc$_{1}$}& \textbf{Acc$_{N/5}$}& \textbf{Time$_{N/5}$}  \\
     \hline
     
     \textbf{D1} & 25,000 &0.4891 & 1.0 & 53.41 ms\\ \hline
     
    %  \textbf{D2} & 303 & 0.4587 & 1.0 & 0.02 ms \\ \hline
     
     \textbf{D2} & 1,372 & 0.4446& 1.0& 0.2 ms \\ \hline
     
     \textbf{D3} & 569 & 0.3448& 1.0& 0.06 ms \\ \hline
     
     \textbf{D4} & 306 & 0.2647& 1.0& 0.03 ms\\ \hline
\end{tabular}
    \label{tab:experiment_no_noise}
\end{table}

\textbf{Results on Real Binary Classification Datasets.} We now discuss (unnoised) label inference on real datasets. The list of datasets we use is as follows: 
% \abhi{Do we really need 5 datasets? Maybe we can save some space by keeping only 3 of them?} 
\begin{CompactEnumerate}
\item \textbf{D1} (IMDB movie review for sentiment analysis~\cite{maas-EtAl:2011:ACL-HLT2011}) -- 0 (negative review) or 1 (positive review); 

% \textbf{D2} -- The Cleveland Heart Disease Dataset. The variable to predict is encoded as 0 to 4 where 0 means no heart disease and 1-4 means presence of heart disease. We denote the latter as 1 in our experiments; 

\item \textbf{D2} (Banknote Authentication) -- 0 (fine) or 1 (forged); 

\item \textbf{D3} (Wisconsin Cancer) -- 0 (benign) and 1 (malignant); 
% The variable to predict is encoded as 2 (benign) or 4 (malignant). 
% We change these to 0 (benign) and 1 (malignant) for our attack; 

\item \textbf{D4} (Haberman’s Survival) -- 0 (survived) and 1 (died). 
% The variable to predict is encoded as 1 (survived) or 2 (died). We change these to 0 (survived) and 1 (died) for our attack. 
\end{CompactEnumerate}

As our attacks construct prediction vectors that are independent of the dataset contents, we ignore the dataset features in our experiments. Our results are summarized in Table~\ref{tab:experiment_no_noise} with $N/5$ log-loss queries. In Figure~\ref{fig:queries_on_D1}, for \textbf{D1}, we plot the results as we double the queries from 1 to $N/5$. Note while the accuracy is low to start with, as soon as we get sufficiently large number of queries we get perfect recovery.

% When using single-query inference, the accuracy is low because of the finite precision in the machine. However, when we carry out the multi-query inference with $N/5$ queries, we achieve 100\% accuracy.  In our experiment, we start with a single query, which does not give 100\% accuracy on \textbf{D1}; since the label inference algorithm runs very fast (see last column in Table ~\ref{tab:experiment_no_noise}), we double the number of queries and achieve 100\% accuracy, which is how we arrive at the value of $N/5$ in the experiments. 

\begin{figure}[!ht]
    \centering
    \includegraphics[width=0.8\columnwidth]{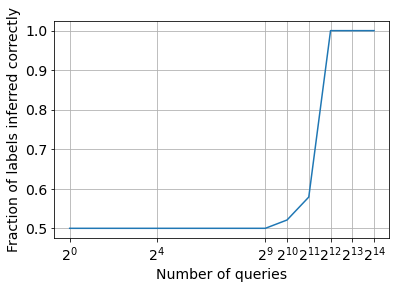}
    \caption{Accuracy of label inference on dataset \textbf{D1} as a function of the number of queries used by the adversary. For $N/5 = 5000$ queries, all labels have been correctly inferred.}
    \label{fig:queries_on_D1}\vspace{0.1in}
\end{figure}
\section{Conclusion}
In this paper, we discussed multiple label inference attacks from log-loss scores. These attacks allow an adversary to efficiently recover all test labels in a small number of queries, even when the observed scores can be erroneous due to perturbation or precision constraints (or both), without any access to the underlying dataset or model training. These results shed light into the amount of information that log-loss scores leak about the test datasets. 

A natural question to ask is if there are ways to defend against such inference attacks. One way is to apply Warner's randomized response mechanism~\cite{warner1965randomized} to the test labels before computing the loss score. This way, when an adversary runs our inference attacks, the labels that it recovers will be protected by plausible deniability. Yet another way is to report the scores on a randomly chosen subset of the test dataset. In this case, bounding the loss is not possible when the size of this subset is unknown. 

\section*{Acknowledgements}
We would like to thank the anonymous reviewers of ICML 2021 for their feedback. We are also grateful to Prakash Krishnamurthy, Shengsheng Liu, Huiming Song and the scientific publications team at Amazon for their helpful comments on the earlier drafts of this paper and providing the necessary resources for this research. 
% \abhi{Add acknowledgements.}

\newpage
\bibliography{ref}
\bibliographystyle{icml2021}

\newpage
\appendix
\clearpage
\newpage
\onecolumn
\appendix
\noindent\rule{\textwidth}{1pt}
\begin{center}
{\Large \textbf{Appendix for ``Label Inference Attacks from Log-loss Scores''}}
\end{center}\vspace{-0.8em}
\noindent\rule{\textwidth}{1pt}

In this technical appendix, we present the missing details and omitted proofs from the main body of our paper.
\section{Missing Details from Section~\ref{sec:rawScores}} \label{app:rawScores}
% We present the omitted proofs and details from Section~\ref{sec:rawScores}.
\subsection{Label Inference under Bounded Precision with Polynomial-Time Adversary} \label{app:boundpoly}
For our results in this section, we make use of the well-known Prime Number Theorem (PNT), which describes the asymptotic distribution of the prime numbers among the positive integers. We use this result to quantify the number of queries required for our label inference attack in the $\mathsf{FPA}(\phi)$ model. The form of PNT which will be helpful for our analysis is stated as follows (we state it as a Lemma for use in our proof for the main result in this section).
\begin{lemma}\label{lem:prime_number_theorem}
\cite{hadamard1896distribution,poussin1897recherches} The $n^{th}$ prime number $p_n$ satisfies $p_n\sim n \log n$, 
where $\sim$ means that the relative error of this approximation approaches $0$ as $n$ increases.
\end{lemma}
We note that this bound on $p_n$ implies that $p_n = \Theta(n \log n)$ also holds. We begin with proving the following lemma.

\begin{lemma}\label{lem:loss_vector_with_ones}
    Let $M \leq N$ be a positive integer and $\mathbf{v} = [v_1,\dots,v_M,1,\dots,1] \in \paran{\mathbb{R}^+}^N$ be a real valued vector. Then, for any labeling $\sigma \in \{0,1\}^N$, it holds that:
	$$\mathcal{L}_{\mathbf{v}}\paran{\sigma} = - \frac{1}{N} \ln \paran{\prod_{\substack{i: \sigma_i = 1\\1 \leq i \leq M}} v_i} + \text{constant},$$
	where the constant term is independent of $\sigma$.
\end{lemma}
	
\begin{proof}
    From the definition of log-loss in Equation~(\ref{eq:loss_expression}), we compute the following:
\begin{align*}
	N\mathcal{L}_{\mathbf{v}}\paran{\sigma} &= - \ln \parfrac{\prod_{i: \sigma_i = 1} v_i}{(1+v_1)\cdots(1+v_N)} \\
	&= - \ln \parfrac{\prod_{\substack{i: \sigma_i = 1\\1 \leq i \leq M}} v_i}{2^{N-M}(1+v_1)\cdots(1+v_M)} = - \ln \paran{\prod_{\substack{i: \sigma_i = 1\\1 \leq i \leq M}} v_i} + \emph{constant},
\end{align*}	
where $\text{\emph{constant}} = (N-M)\ln 2 + \sum_{j=1}^M \ln (1+v_j)$. Dividing both sides by $N$, we obtain the desired result.
\end{proof}

\subsection{Extension to the Multiclass Case} \label{app:multi}

We begin by stating our result for the single-query label inference in the $\mathsf{APA}$ model.

\begin{theorem}\label{thm:multi_class_attack}
	There exists a polynomial-time adversary for $K$-ary label inference in the $\mathsf{APA}$ model using only a single log-loss query.
\end{theorem}
\begin{proof}
	We only focus on $K \ge 3$ (since $K=2$ is equivalent to Theorem~\ref{thm:basic_log_loss_attack}). Let $\sigma^* \in [K]$ be the true labeling. 
	
	Define the matrix $\mathbf{v}$ as follows:
	\begin{align*}
		\mathbf{v}_{i,k} &= \frac{p_i^{k-1}}{\sum_{j=1}^K p_i^{j-1}}\quad \qquad \forall (i, k) \in [N] \times [K].
	\end{align*}
	We can perform the following algebraic manipulation of the log loss (using Definition~\ref{def:log_loss_multi}) as follows:
	\begin{align*}
		\logloss\paran{\mathbf{v}; \sigma^*} &= \frac{-1}{N}\sum_{i=1}^N \sum_{k=1}^K \ \Big([\sigma_i^* = k] \cdot \ln \mathbf{v}_{i,k} \Big) = \frac{-1}{N}\ln \paran{\prod_{i=1}^N \mathbf{v}_{i,\sigma^*_i}}\\ 
		&= \frac{-1}{N}\ln \prod_{i=1}^N \parfrac{p_i^{\sigma^*_i-1}}{\sum_{j=1}^K p_i^{j-1}} = \frac{-1}{N}\ln \parfrac{\prod_{i=1}^N p_i^{\sigma^*_i-1}}{\prod_{i=1}^N \sum_{j=1}^K p_i^{j-1}}.
	\end{align*}
	Rearranging the terms above, we obtain the following:
	\begin{align*}
		\prod_{i=1}^N p_i^{\sigma^*_i-1} &= \exp\paran{-N \cdot \logloss\paran{\mathbf{v}; \sigma^*}} \paran{\prod_{i=1}^N \sum_{j=1}^K p_i^{j-1}}.
	\end{align*}
	Thus, similar to the binary case, from the Fundamental theorem of arithmetic, given the value on the right hand side above, we can uniquely factorize this value to obtain the true labels on the left (which can be done efficiently since the list of primes and the number of classes are known).
\end{proof}

To extend this algorithm for the $\mathsf{FPA}(\phi)$ model, we first make some important observations about using multiple queries for label inference (similar to the binary case). We analyze the case where $K < N$ and focus only on inferring the first set of labels (the analysis for other cases follow using similar principles). We let $m < K$ denote an upper bound on the number of labels that we will infer in a single query (note that previously, $m$ was the exact number of labels inferred per query). Moreover, for simplicity, we will assume that $m$ divides both $K$ and $N$. This will not change the asymptotic number of queries required by our inference attack.

Define the matrices $\mathbf{v}^{(m)}$ and $\mathbf{u}^{(m)}$ as follows:
	\begin{align*}
		\mathbf{v}^{(m)}_{i,k} &= \begin{cases}
		p_i^{k-1} & \forall (i, k) \in [m] \times [m]\\
		1 & \text{otherwise.}
		\end{cases}, \text{ and }\\ \mathbf{u}^{(m)}_{i,k} &= \frac{\mathbf{v}^{(m)}_{i,k}}{\sum_{j=1}^K \mathbf{v}^{(m)}_{i,j}} = \begin{cases}
		\frac{p_i^{k-1}}{\sum_{j=1}^m p_i^{j-1} + (K-m)} & \forall (i, k) \in [m] \times [m]\\
		\frac{1}{\sum_{j=1}^m p_i^{j-1} + (K-m)} & \text{otherwise.}
		\end{cases}
	\end{align*}
Now, substituting this in Definition~\ref{def:log_loss_multi}, we obtain the following:
\begin{align*}
    \logloss\paran{\mathbf{u}^{(m)}; \sigma^*} &= \frac{-1}{N}\sum_{i=1}^N \sum_{k=1}^K \ \Big([\sigma_i^* = k] \cdot \ln \mathbf{u}^{(m)}_{i,k} \Big) \\
    &= \frac{-1}{N}\sum_{i=1}^m \sum_{k=1}^m \ \Big([\sigma_i^* = k] \cdot \ln \mathbf{u}^{(m)}_{i,k} \Big) + \frac{-1}{N}\sum_{i=m+1}^N \sum_{k=m+1}^K \ \Big([\sigma_i^* = k] \cdot \ln \mathbf{u}^{(m)}_{i,k} \Big)\\
    &= \frac{-1}{N}\sum_{i=1}^m \ln \parfrac{p_i^{\sigma_i^*-1} \cdot [\sigma_i^* \leq m]}{\sum_{j=1}^m p_i^{j-1} + (K-m)} - \frac{1}{N}\sum_{i=m+1}^N \ln \parfrac{1}{\sum_{j=1}^m p_i^{j-1} + (K-m)} \\
    &= \frac{1}{N}\sum_{i=1}^m\ln \paran{\sum_{j=1}^m p_i^{j-1} + (K-m)} -\frac{1}{N}\sum_{i=1}^m \ln \paran{p_i^{\sigma_i^*-1}  \cdot [\sigma_i^* \leq m]} + \frac{(N-m)\ln K}{N}\\
    &= \frac{1}{N}\sum_{i=1}^m\ln \paran{\sum_{j=1}^m p_i^{j-1} + (K-m)} -\frac{1}{N} \ln \paran{\prod_{i=1}^m \paran{p_i^{\sigma_i^*-1}  \cdot [\sigma_i^* \leq m]}} + \frac{(N-m)\ln K}{N}.
\end{align*}
Thus, when the score $\ell^{(m)} := \logloss\paran{\mathbf{u}^{(m)}; \sigma^*}$ is observed, the adversary can compute the following:
\begin{align*}
    \prod_{i=1}^m \paran{p_i^{\sigma_i^*-1}  \cdot [\sigma_i^* \leq m]} &= \exp\paran{-N\ell^{(m)}+(N-m)\ln K + \sum_{i=1}^m\ln \alpha(i,m,K)},
\end{align*}
where $\alpha(i,m,K) = \sum_{j=1}^m p_i^{j-1} + (K-m)$. Using the Fundamental Theorem of Arithmetic, the product on the left will allow recovering labels for datapoints that have labels in $[m]$.

\noindent\textbf{Restatement of Theorem~\ref{thm:multi_class_log_loss_attack}.} \emph{There exists a polynomial-time adversary for K-ary label inference in the $\mathsf{APA}$ model using only a single log-loss query. For inference in the $\mathsf{FPA}(\phi)$ model, it suffices to issue $O\paran{1 + NK h(\phi)}$ queries, where $$h(\phi) = O\paran{\frac{(\ln \phi)^2}{\paran{\phi + (N-K)\ln K}^{2/3}}}.$$}
\begin{proof}
	The largest value of $m$ that works in the $\mathsf{FPA}(\phi)$ model can be obtained (asymptotically) by setting $(\phi-1)/2 \geq \log_2 \paran{\prod_{i=1}^m \paran{p_i^{\sigma_i^*-1}  \cdot [\sigma_i^* \leq m]}}$, which will allow enough resolution for this disambiguation. Simplifying this, we obtain:
\begin{align*}
    \phi &\geq 1 + 2\log_2 \paran{\exp\paran{-N\ell^{(m)}+(N-m)\ln K + \sum_{i=1}^m\ln \alpha(i,m,K)}}\\
    &\geq 1+\frac{2}{\ln 2}\paran{(N-K)\ln K + \sum_{i=1}^m \ln \paran{\sum_{j=1}^m p_i^{j-1}}}\\
    &\geq 1+\frac{2}{\ln 2}\paran{(N-K)\ln K + (m-1) \sum_{i=1}^m \ln p_i}\\
    &\gtrsim 3\paran{(N-K)\ln K + (m-1) \sum_{i=1}^m i\ln i}\ \ \ \text{(Using the Prime Number Theorem -- see Lemma~\ref{lem:prime_number_theorem})}\\
    &\geq c\paran{(N-K)\ln K + m^3\ln m} \ \ \ \ \ \text{for some constant $c > 0$.}
\end{align*}
This gives an upper bound on $m$ by simplifying $m^3 \ln m \leq \phi/c - (N-K)\ln K$, which gives $m\ln m \leq \parfrac{\phi/c - (N-K)\ln K}{3}^{1/3}$, or equivalently, $m \leq c' \parfrac{\paran{\phi + (N-K)\ln K}^{1/3}}{\ln \phi}$ for some constant $c' > 0$ (here we use the fact that $x\ln x = y$ holds when $x = \Theta(y/\ln y)$). Setting $m = m^*$, where $m^*$ is this upper bound, we can then obtain the number of queries as $NK/(m^*)^2$, which gives $h(\phi) = 1/(m^*)^2$, or equivalently,
\begin{align*}
	h(\phi) = O\paran{\frac{(\ln \phi)^2}{\paran{\phi + (N-K)\ln K}^{2/3}}}.
\end{align*}
\end{proof}

We defer optimization of this algorithm and the detailed discussion of Robust Label Inference for multi-class classification to future work.

\section{Missing Details from Section~\ref{sec:bounded_noise}} \label{app:bounded_noise}
% We present the omitted proofs from Section~\ref{sec:bounded_noise}. We begin by proving two helper lemmas.

\subsection{$\tau$-Robust Label Inference under Arbitrary Precision with Exponential-time Adversary}\label{app:robust_arbitrary_exponential}

\noindent\textbf{Restatement of Lemma~\ref{lem:leaky_sensitivity_helper}.} \emph{Let $\mathbf{v} = [v_1,\dots,v_N]$ be a vector with all entries distinct and positive. Define $\ln \mathbf{v} := [\ln v_1,\dots,\ln v_N]$. Then, it holds that $\Delta(\mathbf{v}) = \frac{1}{N}\mu(\ln \mathbf{v})$.}
\begin{proof}
Without loss of generality, assume that $\sigma_1$ and $\sigma_2$ are such that $\mathcal{L}_{\mathbf{v}}\paran{\sigma_1} \ge \mathcal{L}_{\mathbf{v}}\paran{\sigma_2}$. This happens when the following holds (from Equation~(\ref{eq:loss_expression})):
\begin{align*}
	\mathcal{L}_{\mathbf{v}}\paran{\sigma_1} \ge \mathcal{L}_{\mathbf{v}}\paran{\sigma_2} \Longleftrightarrow \prod_{i: \sigma_1(i) = 1} v_i \le \prod_{j: \sigma_2(j) = 1} v_j.
%	&\Longleftrightarrow \frac{-1}{N} \ln \parfrac{\prod_{i: \sigma_1(i) = 1} v_i}{(1+v_1)\dots(1+v_N)} > \frac{-1}{N} \ln \parfrac{\prod_{j: \sigma_2(j) = 1} v_j}{(1+v_1)\dots(1+v_N)} \\
%	\Longleftrightarrow \prod_{i: \sigma_1(i) = 1} v_i &< \prod_{j: \sigma_2(j) = 1} v_j.
\end{align*}
Under this condition, we can write the following:
\begin{align*}
	N\Delta(\mathbf{v}) &= \min_{\sigma_1, \sigma_2 \in \{0,1\}^N} N\paran{\mathcal{L}_{\mathbf{v}}\paran{\sigma_1} - \mathcal{L}_{\mathbf{v}}\paran{\sigma_2}} = \min_{\substack{\sigma_1, \sigma_2 \in \{0,1\}^N \\ \sigma_1 \neq \sigma_2}} \ln \exp\paran{N\mathcal{L}_{\mathbf{v}}\paran{\sigma_1} - N\mathcal{L}_{\mathbf{v}}\paran{\sigma_2}} \\
	&= \min_{\sigma_1, \sigma_2 \in \{0,1\}^N} \ln \parfrac{\exp\paran{-N\mathcal{L}_{\mathbf{v}}\paran{\sigma_2}}}{\exp\paran{-N\mathcal{L}_{\mathbf{v}}\paran{\sigma_1}}} = \min_{\sigma_1, \sigma_2 \in \{0,1\}^N} \ln \paran{\frac{\prod_{j: \sigma_2(j) = 1} v_j}{\prod_{i: \sigma_1(i) = 1} v_i}} \\
	&= \min_{\sigma_1, \sigma_2 \in \{0,1\}^N} \paran{\sum_{j: \sigma_2(j) = 1} \ln v_j - \sum_{i: \sigma_1(i) = 1} \ln v_i} = \mu\paran{\ln \mathbf{v}},
\end{align*}
% \zekun{Can we add a simple additional step between the last and the second last equation in the derviation above to show the equivalence with definition 2? Because $S_1$ and $S_2$ in Definition 2 are arbitrary subsets where the subsets in the derivation appear to be for labels = 1 (I believe they are equivalent but it would be nice to show it.)}
where the last step follows from interpreting $\sigma_1$ and $\sigma_2$ as selecting elements from $\ln \mathbf{v}$ (by choosing which elements get labeled 1 and the ones that do not).
\end{proof}

\noindent\textbf{Restatement of Theorem~\ref{thm:tau_leaky}.} \emph{For any $\tau > 0$, there exists an exponential-time adversary (from Algorithm~\ref{alg:bounded_noise_exponential}) for the $\tau$-robust label inference problem in the \textsf{APA} model using only a single log-loss query.}
\begin{proof}
Let $S = \{s_1,\dots,s_N\}$ be a set such that $\mu(S) \geq 2N\tau$. We know that $S$ exists, for example, by scaling each element of $\mathcal{S}_\circ$ by $2N\tau$. Let $\mathbf{v}$ be constructed such that $v_i = 3\exp(s_i)$. To see that $\mathbf{v}$ is $\tau$-robust, it suffices to prove that $\Delta(\mathbf{v})> 2\tau$. Now, observe that $(\ln 3)s_i = \ln \mathbf{v}_i$. From Lemma~\ref{lem:leaky_sensitivity_helper}, we get  $\Delta(\mathbf{v}) = \frac{\ln 3}{N}\mu(S) = (2\ln 3)\tau > 2\tau$. 
\end{proof}

\subsection{Optimality of Single-Query $\tau$-Robust Inference} \label{app:opti}

We begin with the following two lemmas.
\begin{lemma}~\label{lem:min_distance}
A vector $\mathbf{v}$ is  $\tau$-robust if and only if $\Delta(\mathbf{v}) > 2\tau$.
\end{lemma}
\begin{proof}
It suffices to prove that for any $\Delta(\mathbf{v}) > 2\tau$ for any $\tau$-robust vector $\mathbf{v}$ (since the other direction follows from our construction in Algorithm~\ref{alg:bounded_noise_exponential}). We prove this by contradiction. The idea is to construct a score from which a unique labeling cannot be unambiguously derived. Let $\mathbf{v}$ be a $\tau$-robust vector and $\mathcal{A}$ be a Turing Machine such that $\mathcal{A}(\ell,N,\tau,\mathbf{v}) = \sigma$ for all $\sigma \in \{0,1\}^N$ and all $\ell$ that satisfies $\abs{\ell - \mathcal{L}_{\mathbf{v}}\paran{\sigma}} \leq \tau$. Without loss of generality, let $\sigma_1, \sigma_2$ be two distinct labelings for which $0 < \mathcal{L}_{\mathbf{v}}\paran{\sigma_2} - \mathcal{L}_{\mathbf{v}}\paran{\sigma_1} < 2\tau$. It follows that $\mathcal{L}_{\mathbf{v}}\paran{\sigma_2} - \tau < \mathcal{L}_{\mathbf{v}}\paran{\sigma_1} + \tau$ (see schematic below).  
\vspace{-1em}    
\begin{center}
 \begin{tikzpicture}
\path [draw=black,fill=black] (-3.5,0) circle (1.5pt);
\path [draw=black,fill=black] (-1.3,0) circle (1.5pt);
\path [draw=black,fill=black] (0,0) circle (1.5pt);
\path [draw=black,fill=black] (1.3,0) circle (1.5pt);
\path [draw=black,fill=black] (3.5,0) circle (1.5pt);
\draw[latex-latex] (-4,0) -- (4,0) ;
\foreach \x in  {-3.5,-1.3,1.3,3.5}
\draw[shift={(\x,0)},color=black] (0pt,3pt) -- (0pt,-3pt);
\draw[shift={(-3.5,0)},color=black] (0pt,0pt) -- (0pt,-3pt) node[below] 
{$\mathcal{L}_{\mathbf{v}}\paran{\sigma_1}$};
\draw[shift={(-1.3,0)},color=black] (0pt,0pt) -- (0pt,-3pt) node[below] 
{$\mathcal{L}_{\mathbf{v}}\paran{\sigma_2} - \tau$};
\draw[shift={(0,0)},color=black] (0pt,0pt) -- (0pt,-3pt) node[below] 
{$\ell$};
\draw[shift={(1.3,0)},color=black] (0pt,0pt) -- (0pt,-3pt) node[below] 
{$\mathcal{L}_{\mathbf{v}}\paran{\sigma_1} + \tau$};
\draw[shift={(3.5,0)},color=black] (0pt,0pt) -- (0pt,-3pt) node[below] 
{$\mathcal{L}_{\mathbf{v}}\paran{\sigma_2}$};

\draw[decorate, decoration={brace, amplitude=5pt,mirror}] (-3.5,-0.75)--(0,-0.75) node[midway,yshift=-0.4cm] {$x$};
\draw[decorate, decoration={brace, amplitude=5pt}] (-3.5,0.3)--(3.5,0.3) node[midway,yshift=0.4cm] {$< 2\tau$};
\end{tikzpicture} 
\end{center}
\vspace{-1.2em}
Let $\ell = \paran{\mathcal{L}_{\mathbf{v}}\paran{\sigma_1}+\mathcal{L}_{\mathbf{v}}\paran{\sigma_2}}/2$ and $x = \ell -  \mathcal{L}_{\mathbf{v}}\paran{\sigma_1}$. Clearly, $x < \tau$, and thus, $\mathcal{A}\paran{\ell,N,\tau,\mathbf{v}} = \sigma_1$. However, we can show similarly that $\mathcal{L}_{\mathbf{v}}\paran{\sigma_2} - \ell < \tau$, and hence $\mathcal{A}\paran{\ell,N,\tau,\mathbf{v}} = \sigma_2$. This is a contradiction since $\sigma_1 \neq \sigma_2$, and hence, it must be true that $\mathcal{L}_{\mathbf{v}}\paran{\sigma_2} - \mathcal{L}_{\mathbf{v}}\paran{\sigma_1} > 2\tau$. The result in the Lemma statement then follows by observing that if this condition holds for any $\sigma_1, \sigma_2$, then  $\Delta(\mathbf{v}) = \min_{\sigma_1,\sigma_2}\abs{\mathcal{L}_{\mathbf{v}}\paran{\sigma_2} - \mathcal{L}_{\mathbf{v}}\paran{\sigma_1}} > 2\tau$ as well. 
\end{proof}

\begin{lemma}\label{lem:set_construct}
For every $\tau$-robust vector $\mathbf{v} = [v_1,\dots,v_N]$ with distinct entries in $(1,\infty)$, it is possible, for every $s > 0$, to construct a set $S$ such that $\mu(S) > s$.
\end{lemma}
\begin{proof}
    From Lemma~\ref{lem:leaky_sensitivity_helper}, we know that $\Delta(\mathbf{v}) = \mu(\ln \mathbf{v})/N$. Since $\mathbf{v}$ is $\tau$-robust, it must be true that $\Delta(\mathbf{v}) > 2\tau$ (by Lemma~\ref{lem:min_distance}), from which we obtain that $\mu(\ln \mathbf{v}) > 2N \tau$. The result in the lemma statement then follows by setting $S = \{s_1,\dots,s_N\}$, where $s_i = (s\ln v_i)/(2N\tau)$.
\end{proof}

\noindent\textbf{Restatement of Theorem~\ref{thm:euler_general}.} \emph{For any set $S \subset \mathbb{Q}^{+}$ with $\mu(S) > \lambda$ for some $\lambda \in [0,\infty)$, it holds that $\pnorm{\infty}{S} = \Omega(\lambda 2^{|S|})$.}
\begin{proof}
    We prove this result in three steps. First, we show that the bound holds whenever $\lambda$ as well as all the elements in $S$ are positive integers. To see this, observe that if $\mu(S) > 1 > 0$, Euler's result gives us that $\pnorm{\infty}{S} = \Omega(2^{|S|})$. Now, let $S' = \{s\lambda \ | \ s \in S\}$. Then, $\mu(S') > \lambda$. If we suppose that $\pnorm{\infty}{S'} = o(\lambda 2^{|S|})$, then $\pnorm{\infty}{S'/T} = \pnorm{\infty}{S} = o(2^{|S|})$, which is a contradiction to above. 
    
    Next, assume that $S \subset \mathbb{Q}^{+}$ and $\lambda$ still be an integer. In this case, each element of $S$ can be written as $s_i = p_i/q_i$ (in the lowest form). Let $Q = q_1q_2\cdots q_N$ and $S' = \{sQ \ | \ s \in S\}$. Then, each element of $S'$ is an integer and since $\mu(S') > \lambda Q$, which is also an integer, from the discussion above, we have $\pnorm{\infty}{S'} = \Omega(\lambda Q 2^N)$, which gives $\pnorm{\infty}{S} = \Omega(\lambda 2^N)$.
    
    Finally, let $\lambda$ be an arbitrary positive real. In this case, $\mu(S) > \lambda$ implies $\mu(S) > \lfloor \lambda \rfloor$, which is an integer. Hence, $\pnorm{\infty}{S} = \Omega(\lfloor \lambda \rfloor 2^N) = \Omega(\lambda 2^N)$, as desired.
\end{proof}

\noindent\textbf{Restatement of Theorem~\ref{thm:lower_bound_helper}:} \emph{For sufficiently large $N$ and all $\tau > 0$, any $\tau$-robust vector $\mathbf{v}$ must have $\pnorm{\infty}{\mathbf{v}} = \Omega\paran{e^{2^N N \tau}}$.}
\begin{proof}
From Lemma~\ref{lem:leaky_sensitivity_helper}, we know that for any vector $\mathbf{v}$ with distinct positive entries, it holds that $\Delta\paran{\mathbf{v}} = \mu\paran{\ln \mathbf{v}}/N$. For this vector to be $\tau$-robust, we argue that $\Delta\paran{\mathbf{v}} > 2\tau$ (see Lemma~\ref{lem:min_distance}), which is the same as setting $\mu\paran{\ln \mathbf{v}} > 2N\tau$. From Theorem~\ref{thm:euler_general}, for this to hold, it must be true that $\pnorm{\infty}{\ln \mathbf{v}} = \Omega\paran{2^N N \tau}$. From the definition of $\ln \mathbf{v}$, this gives $\pnorm{\infty}{\mathbf{v}} = \Omega\paran{\exp\paran{2^N N \tau}}$, as desired.
\end{proof}

\subsection{$\tau$-Robust Label Inference under Bounded Precision with Polynomial-time Adversary}\label{app:robust_precision_polynomial}

We prove an additional lemma before proving our next result.

\begin{lemma}\label{lem:partial_vectors}
	Let $\tau > 0$ be a bound on the resulting error and $m \leq N \in \mathbb{Z}^{+}$ be an integer. Let $\mathbf{v}_m = \left[ 3e^{2m\tau}, 3e^{4m\tau}, \dots, 3e^{2^{m}m\tau}, 1, \dots, 1 \right]$. Then, for any distinct $\sigma_1, \sigma_2 \in \{0,1\}^N$ where $\sigma_1[:m] \neq \sigma_2[:m]$ (\emph{i.e.} $\sigma_1$ and $\sigma_2$ differ some index $i \leq m$), it holds that: $$\left| \mathcal{L}_{\mathbf{v}_m}\paran{\sigma_1} - \mathcal{L}_{\mathbf{v}_m}\paran{\sigma_2} \right| > 2m\tau/N.$$
\end{lemma}
\begin{proof}
	From Lemma~\ref{lem:loss_vector_with_ones}, we can write that:
\begin{align*}
	\abs{\mathcal{L}_{\mathbf{v}_m}\paran{\sigma_1} - \mathcal{L}_{\mathbf{v}_m}\paran{\sigma_2}} &= \frac{1}{N}\abs{\sum_{\substack{i: \sigma_1(i) = 1\\1 \leq i \leq m}} \ln v_m(i) - \sum_{\substack{k: \sigma_2(k) = 1\\1 \leq k \leq m}} \ln v_m(k)} \\
	&= \frac{
	\ln 3}{N}\abs{\sum_{\substack{i: \sigma_1(i) = 1\\1 \leq i \leq m}} 2^{i}m\tau - \sum_{\substack{k: \sigma_2(k) = 1\\1 \leq k \leq m}} 2^{k}m\tau} = \frac{m\tau \ln 3}{N} \abs{\sum_{\substack{i: \sigma_1(i) = 1\\1 \leq i \leq m}} 2^{i} - \sum_{\substack{k: \sigma_2(k) = 1\\1 \leq k \leq m}} 2^{k}} > \frac{2m\tau}{N},
\end{align*}
where the last step follows from the fact that $\sigma_1$ and $\sigma_2$ differ in at least one element amongst the first $m$ entries.
\end{proof}

\textbf{Restatement of Lemma~\ref{cor:vec_attack}.} \emph{Let $\tau > 0$ be a bound on the resulting error and $m \leq N \in \mathbb{Z}^{+}$ be an integer. Let $\mathbf{v}_m = \left[ 3e^{2N\tau}, 3e^{4N\tau}, \dots, 3\exp\paran{2^{m}N\tau}, 1, \dots, 1 \right]$.
	Then, for any distinct $\sigma_1, \sigma_2 \in \{0,1\}^N$, the following hold:
	\begin{enumerate}
	    \item[(a)] If $\sigma_1[:m] = \sigma_2[:m]$, then $\mathcal{L}_{\mathbf{v}_m}\paran{\sigma_1} = \mathcal{L}_{\mathbf{v}_m}\paran{\sigma_2}$.
	    \item[(b)] Else, we have $\left| \mathcal{L}_{\mathbf{v}_m}\paran{\sigma_1} - \mathcal{L}_{\mathbf{v}_m}\paran{\sigma_2} \right| > 2\tau$.
	\end{enumerate}}
The proof directly follows from Lemma~\ref{lem:partial_vectors}.

\noindent\textbf{Restatement of Theorem~\ref{thm:proof_reconstruction}.} \emph{For any error bounded by $\tau > 0$ and $\phi \geq 8 + \lceil N\tau \ln 2 \rceil$, there exists a polynomial-time adversary (from Algorithm~\ref{alg:reconstruction_attack}) for the $\tau$-label inference problem in the $\mathsf{FPA}(\phi)$ model using $O\paran{\frac{N}{\log N} + \frac{N}{\log \paran{\phi/N\tau}}}$ queries.}
\begin{proof}
To see how many queries suffice, we first compute the number of bits necessary to represent the prediction vector and the loss scores up to sufficient resolution. For $\mathbf{u}_m = f(\mathbf{v}_m)$, we observe the following for sufficiently large $N$, in particular, when $N\tau > \ln(16)$:  
\begin{align*}
    \min_{\substack{i,j\in [N]\\\mathbf{u}_m(i)\neq \mathbf{u}_m(j)}}\abs{\mathbf{u}_m(i) - \mathbf{u}_m(j)} &= \frac{3e^{2^m N\tau}}{1+3e^{2^m N\tau}} - \frac{3e^{2^{m-1} N\tau}}{1+3e^{2^{m-1} N\tau}} \\ 
    &\geq \frac{\exp\paran{-2^{m-1}N\tau}}{16},
\end{align*}
which, to work within the $\mathsf{FPA}(\phi)$ model, would require $(\phi-1)/2 \geq 4 + 2^{m-1}N\tau \ln 2$, or equivalently, $\phi \geq 8 + 2^{m}N\tau \ln 2$ bits. Moreover, we can bound the loss computed on $\mathbf{v}_m$ on any $\sigma$ as follows:
\begin{align*}
    \max_{\sigma \in \{0,1\}^N} \mathcal{L}_{\mathbf{v}_m}(\sigma) &\leq \frac{1}{N}\ln \paran{2^{N-m}\prod_{j=1}^m (1+3e^{2^j N\tau})}\\
    &\leq 2\paran{\ln 2 + (2^m - 1)\tau} \leq 2^{m+1}\tau,
\end{align*}
which can be represented (along with sufficient resolution, since $\Delta(\mathbf{v}_m) \geq 2\tau$ by construction) within the number of bits described above. Thus, in the $\mathsf{FPA}(\phi)$ model, the maximum value of $m$ that we can set is obtained by setting $2^{m}N\tau \ln 2 + 8 \leq \phi$, which gives: $$m_{max} = \left\lfloor \log_2 \parfrac{\phi -8}{N \tau \ln 2} \right\rfloor.$$From this, we obtain that $\frac{N}{m_{max}} = O\paran{\frac{N}{\log N} + \frac{N}{\log \paran{\phi/N\tau}}}$ queries suffice.

Given this bound, it suffices to show that in the $i^{th}$ iteration of the for-loop in Algorithm~\ref{alg:reconstruction_attack}, the vector $\sigma'$ contains the true labels for indices in $\{(i-1)m+1,\dots,im\}$.  Without loss of generality, assume $i=1$, so that the vector $\mathbf{v} = \left[ 3e^{2N\tau}, 3e^{4N\tau}, \dots, 3\exp\paran{2^{m}N\tau}, 1, \dots, 1 \right]$.
    From Lemma~\ref{cor:vec_attack}(b), it follows that for any distinct $\sigma_1, \sigma_2 \in \{0,1\}^N$ where $\sigma_1[:m] \neq \sigma_2[:m]$, it holds that:
    \begin{align}
        \label{eq:bound_on_loss}
        \left| \mathcal{L}_{\mathbf{v}}\paran{\sigma_1} - \mathcal{L}_{\mathbf{v}}\paran{\sigma_2} \right| > 2\tau.
    \end{align}
    Thus, when the loss $\ell$ is observed on $\mathbf{u} = f(\mathbf{v})$ and $\arg\min_\sigma \abs{\mathcal{L}_{\mathbf{u}}\paran{\sigma} - \ell} = \{\sigma^{(1)},\dots,\sigma^{(k)}\}$, then it must be true that $\sigma^{(k_1)}[:m] = \sigma^{(k_1)}[:m]$ for all $k_1,k_2 \in [k]$ (or else, it would contradict the inequality in~(\ref{eq:bound_on_loss})). Furthermore, from Lemma~\ref{cor:vec_attack}(a), it follows that the true labeling must be present in the set $\{\sigma^{(1)},\dots,\sigma^{(k)}\}$. Thus, at the end of iteration $i=1$, the vector $\hat{\sigma}$ has recovered the first $m$ bits in $\sigma$. For all other iterations, note that the vector $\mathbf{v}$ is just a cyclic rotation to the right by $m$ elements, and hence, the bits in $\sigma$ are recovered, $m$ at a time, in Algorithm~\ref{alg:reconstruction_attack}. 
\end{proof}

\begin{figure*}[t]
    \centering
    \includegraphics[width=380pt,trim={0 0.18cm 0 0},clip]{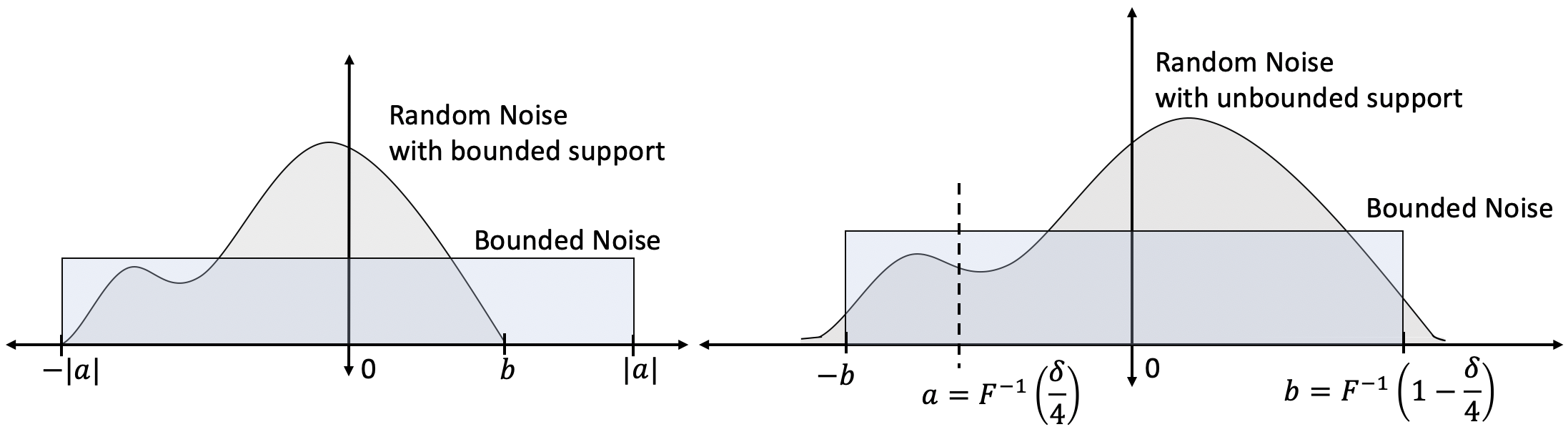}
    \caption{Schematic of the reduction of the random noise case to bounded noise for label inference. The picture on the left allows for a construction of a robust vector, whereas the picture on the right allows for robust vectors.}
    \label{fig:noise_schematic}
\end{figure*}

\subsection{Extension to Other Noise Models: Random Noise} \label{app:random_noise}
The discussion of norm-bounded noise above allows easy extension to handling randomly generated (additive) noise. We achieve this by safeguarding against the worst case magnitude of the noise that can be added for bounded noise distributions. For cases where this distribution is unbounded, we allow for some error tolerance. We begin with formally defining the label inference problem in this setting and then present our analysis. We will restrict ourselves to arbitrary precision in this section and defer the extension to $\mathsf{FPA}(\phi)$ for future work. 

\begin{definition}
Let $\mathcal{D}$ be a probability distribution and $\delta \in [0,1]$ be the error tolerance. We say that a vector $\mathbf{v}$ is $\mathcal{D}$-robust if for all $\tau \sim \mathcal{D}$ and $\sigma \in \{0,1\}^N$, there exists an adversary (Turing Machine) $\mathcal{A}$ that can recover (within a single query) $\sigma$ from $\ell = \mathcal{L}_{\mathbf{v}}\paran{\sigma} + \tau$ with probability at least $1 - \delta$, \emph{i.e.} $\Pr\vect{\mathcal{A}\paran{\ell, N, \mathcal{D},\mathbf{v}} = \sigma} \geq 1 - \delta$.
\end{definition}

Our main result in the section is stated in the theorem below (see Figure~\ref{fig:noise_schematic} for a proof sketch). 
\begin{theorem}\label{thm:random_noise}
	For any error tolerance $\delta \in (0,1)$, it is possible to construct a $\mathcal{D}$-robust vector for all distributions $\mathcal{D}$.
\end{theorem}
\begin{proof}
    We first look at the case when $\mathcal{D}$ has bounded support over $\mathbb{R}$. Let $[a,b] \subset \mathbb{R}$ be the support of $\mathcal{D}$. Let $\tau = \max\{|a|,|b|\}$. Then, for any $\eta \sim \mathcal{D}$, it holds that $\abs{\paran{\mathcal{L}_{\mathbf{v}}\paran{\sigma} + \eta} - \mathcal{L}_{\mathbf{v}}\paran{\sigma}} = \abs{\eta} \leq \tau$ and hence, any $\tau$-robust vector is also $\mathcal{D}$-robust.
	
	For the case when $\mathcal{D}$ has unbounded support, let $\eta \sim \mathcal{D}$, and $a,b \in \mathbb{R}$ be such that $\Pr\paran{\eta \in (-\infty,a) \cup (b,\infty)} < \delta$. To see that this can be done, let $F: \mathbb{R}\to[0,1]$ be the cumulative distribution function for $\mathcal{D}$, which is a nondecreasing, right-continuous function with $F(-\infty)=0$, $F(\infty)=1$,
	and $F^{-1}(p)=\inf\{x\in\mathbb{R}:F(x)\geq p\}$. 
    For any fixed $\delta\in(0,1)$,
    we can set $a=F^{-1}(\delta/4)$ and $b=F^{-1}(1-\delta/4)$ so that
    $\Pr\paran{\eta \in [a,b]} \geq \delta/2$, which makes 
    $\Pr\paran{\eta \in (-\infty,a) \cup (b,\infty)} \leq \delta/2 < \delta$.
    Note that $a$ and $b$ are always finite by definition of the cumulative distribution function as the point mass is zero on $\pm\infty$.
\end{proof}

To see why this works, observe that for any bounded distribution, say over some interval $[a,b] \subset \mathbb{R}$, the amount of noise added is never more than $\max\{|a|,|b|\}$. Thus, any $\max\{|a|,|b|\}$-robust vector can unambiguously recover the labels from the noised scores. For distributions with unbounded support, however, such an upper bound does not exist. Given the error tolerance, it is possible to compute this bound for any distribution and robustness be defined accordingly. We explain this for subexponential noise below. 

Recall that a random variable $X \in \mathbb{R}$ is said to be subexponential (denoted $X \sim \textsf{subE}(\lambda^2, \nu)$) with parameters $\lambda^2, \nu > 0$ if $\mathbb{E}(X)=0$ and    its moment generating function satisfies $\mathbb{E}(e^{sX}) \leq \exp{\paran{\lambda^2s^2/2}}$ for all $|s| < 1/\nu$. Given any tolerance $\delta \in (0,1)$, it can be shown that there exists a $\textsf{subE}(\lambda^2,\nu)$-robust vector for all $\lambda,\nu > 0$. We formally state this result below. 

\begin{theorem}\label{thm:sub_exp_noise}
	For all $\lambda,\nu > 0$ and $\delta \in (0,1)$, any vector that is $\paran{2(\lambda+\nu)\sqrt{\ln \parfrac{2}{\delta}}}$-robust is  $\textsf{subE}(\lambda^2,\nu)$-robust as well. In particular, with probability at least $1-o(1)$, there exists an adversary for the single query Robust Label Inference problem for $\textsf{subE}(\lambda^2,\nu)$ noise.
\end{theorem}
\begin{proof}
    We begin with reminding the reader a concentration bound that will be useful in our proof.
\begin{lemma}[\cite{dubhashi2009concentration}]\label{lem:sub_exponential_bound} Let $Y \sim \textsf{subE}(\lambda^2,\nu)$ be a zero-mean random variables for some $\lambda,\nu > 0$. Then, for all $t > 0$, the following holds: $$\Pr \paran{|Y| > t} \leq 2\exp{\paran{-\frac{1}{2}\paran{\frac{t^2}{\lambda^2} \land \frac{t}{\nu}}}},$$where $a \land b := \min\{a,b\}$.
\end{lemma}

Given this result, we follow the general construction shown in the proof of Theorem~\ref{thm:random_noise}, we compute $a > 0$ such that for any $Y \sim \textsf{subE}(\lambda^2,\nu)$, it holds that $\Pr\paran{Y \in (-a,a)} \ge 1-\delta$. This would allow all $a$-robust vectors to be robust with probability at least $1-\delta$, as needed.

To compute $a$, observe that from Lemma~\ref{lem:sub_exponential_bound}, we can set $t=a$ and set the bound on the probability to be at most $\delta$, as follows:
\begin{align*}
	\Pr \paran{|Y| > a} \leq 2\exp{\paran{-\frac{1}{2}\paran{\frac{a^2}{\lambda^2} \land \frac{a}{\nu}}}} < \delta.
\end{align*}

A simple algebraic manipulation for the inequality on the right gives the condition that $\frac{a^2}{\lambda^2} \land \frac{a}{\nu} > 2\ln \parfrac{2}{\delta}$. We solve the two cases here separately. If $a < \lambda^2/\nu$, then $\frac{a^2}{\lambda^2} \land \frac{a}{\nu} = \frac{a^2}{\lambda^2}$, and hence, this gives $a < \lambda \sqrt{2\ln \parfrac{2}{\delta}}$. Else, we have  $\frac{a^2}{\lambda^2} \land \frac{a}{\nu} = \frac{a}{\nu}$, which gives $a < 2\nu \ln \parfrac{2}{\delta}$. It suffices to take the maximum of these two limits and hence, $\max\{\lambda \sqrt{2\ln \parfrac{2}{\delta}}, 2\nu \ln \parfrac{2}{\delta}\}$ works. We can further simplify this expression using the  upper bound $2(\lambda+\nu)\sqrt{\ln \parfrac{2}{\delta}}$ which follows from the fact that $\delta \in (0,1)$.
\end{proof}

For example, let $Z = \mathcal{N}(0,1)$ denote the standard normal random variable. Then, we know that the Chi-squared random variable with one degree of freedom follows the law for $Z^2 = \textsf{subE}(2,4)$. Thus, from the Corollary above, we obtain that for with probability at least $1-o(1)$, any vector that can handle up to $12\sqrt{\ln N}$ amount of bounded noise can also handle $Z^2$ noise. If we are allowed up to, say $\delta = 0.1$, then any $12\sqrt{\ln 20} \approx 20.77$-robust vector suffices.

\subsection{Extension to Other Noise Models: Multiplicative Noise}\label{app:multiplicative_noise}
We briefly explore extending the analysis above to the case of multiplicative noise. In this case, the adversary observes score $\ell$ that satisfies: $$(1-\alpha_1)\mathcal{L}_\mathbf{v}(\sigma) \le \ell \le (1+\alpha_2)\mathcal{L}_\mathbf{v}(\sigma),$$where $\alpha_1 \in [0,1)$ and $\alpha_2 \ge 0$ are the known bounds on the rate of the multiplicative noise. We show that if the rate is small enough, then it is possible to reduce this case to that of additive noise. To see this, we begin with rewriting the inequality above as  $\abs{\ell - \mathcal{L}_\mathbf{v}(\sigma)} \leq \alpha^* \cdot \abs{\mathcal{L}_\mathbf{v}(\sigma)}$, where $\alpha^* = \max\{\alpha_1, \alpha_2\}$. Now, if we use the construction of a $\tau$-robust vector from Algorithm~\ref{alg:bounded_noise_exponential}, it is easy to compute that $\abs{\mathcal{L}_\mathbf{v}(\sigma)} \leq 2^{N+1}\tau$ for any $\sigma$, which would require ensuring $2^{N+1}\tau \alpha^* < 2\tau$ -- this is not possible for any $\tau > 0$ unless $\alpha^* < 2^{-N}$. Thus, we require multiple queries. 

Suppose we only infer $1 \leq m \leq N$ labels in a single query (using $M=m$ in Algorithm~\ref{alg:reconstruction_attack}), then this will require $\paran{\ln 2 + 2^{m+1}\tau} \alpha^*$ to be at most $\tau$. It suffices to set $\alpha^* \leq 1/2^{m+2}$. Then, the quantity on the left becomes at most $$\frac{\ln 2 + 2^{m+1}\tau}{2^{m+2}} = \frac{\ln 2}{2^{m+2}} + \frac{\tau}{2} \le \tau,$$whenever $\tau \ge \frac{\ln 2}{2^{m+1}}$. We state this formally as follows. 
\begin{theorem}\label{thm:multiplicative_noise}
    For any $\alpha^* \leq \frac{1}{8}$, label inference can be done, $\left\lceil \log_2 \parfrac{1}{\alpha^*} - 2 \right\rceil$ labels at a time, using vectors that are $(2\ln 2)\alpha^*$-robust.
\end{theorem}

Observe that when $\alpha \geq \frac{1}{4}$, then no value of $\tau$ satisfies the constraint above, implying that robust vectors from Algorithm~\ref{alg:reconstruction_attack} cannot be used with any number of queries. The noise is more than what these vectors can guarantee handling. We defer label inference for this case to future work.

\section{Experimental Results for Multi-Class Label Inference}\label{app:experiments_multi_class}
\begin{figure}[t]
    \centering
    \subfloat{\includegraphics[width=185pt]{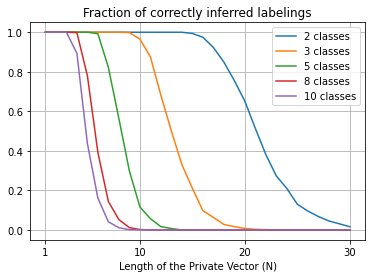}}
    \hspace{1em}
    \subfloat{\includegraphics[width=185pt]{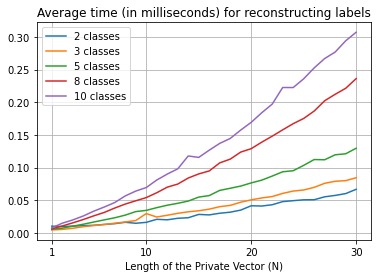}}
    \caption{Empirical results for multi-class label inference. The plot on the top shows that due to fixed-precision, the accuracy drops to zero faster as the number of classes increases, since larger primes powers are used in the prediction vectors. The plot below shows that the inference time increases with the number of classes.\vspace{1em}}
    \label{fig:plots_multi_class}
\end{figure}

Similar to Section~\ref{sec:expts}, we evaluate our attacks on simulated labelings (with no noise) in the multi-class setting (see Figure~\ref{fig:plots_multi_class}). The results show that our algorithms are efficient, even with a large number of datapoints. All experiments are run on a 64-bit machine with 2.6GHz 6-Core processor, using the standard IEEE-754 double precision format (1 bit for sign, 11 bits for exponent, and 53 bits for mantissa). For ensuring reproducibility, the entire experiment setup is submitted as part of the supplementary material. 

The plot on the left in Figure~\ref{fig:plots_multi_class} shows the accuracy of label inference, where we use the attack based on primes from Theorem~\ref{thm:multi_class_attack}. The accuracy reported is with respect to 100 randomly generated labelings for each $N$ (length of the vector to be inferred). We vary the number of classes from $2$ to $10$ for these experiments. Similar to the results in Figure~\ref{fig:experiment_plots}, the maximum accuracy falls to zero when $N$ increases, because of the limited floating point precision on the machine. As the number of classes increase, this drop happens sooner (\emph{i.e.} for a smaller $N$), because the powers of primes get larger. 

The plot on the right shows the run time for our attack. The results show that this inference happens in only a few milliseconds.  

\end{document}